%% file: arxiv.tex
\pdfoutput=1
\documentclass{article}
\usepackage{times,authblk}

\usepackage{geometry}
\usepackage{rotating, color,subfigure,algorithm,algorithmic}

\usepackage{multirow}
\usepackage{graphicx}
\usepackage{hyperref}                                                                                                                                                                                                       
\usepackage{enumerate}
\usepackage{amsthm,amsmath,amssymb}

\usepackage{mkolar_definitions}

\usepackage{ifthen}

\newcommand{\phat}{\hat{p}}
\newcommand{\qhat}{\hat{q}}
\newcommand{\fhat}{\hat{f}}
\newcommand{\wT}{\widehat{T}}

\newcommand{\Thatlin}{\widehat{T}_{lin}}
\newcommand{\Thatplugin}{\widehat{T}_{pl}}
\newcommand{\Thatquad}{\widehat{T}_{quad}}

\newcommand{\yrcite}{\cite}
\allowdisplaybreaks

\newenvironment{packed_enum}{
\begin{enumerate}
\setlength{\itemsep}{1pt}
\setlength{\parskip}{0pt}
\setlength{\parsep}{0pt}
}{\end{enumerate}}

\newcommand{\version}{arxiv}

\begin{document}
\title{Nonparametric Estimation of R\'{e}nyi Divergence and Friends}

\author[1]{
Akshay Krishnamurthy
\thanks{akshaykr@cs.cmu.edu}}

\author[2]{
Kirthevasan Kandasamy
\thanks{kandasamy@cs.cmu.edu}}

\author[2]{
\\Barnab\'{a}s P\'{o}czos
\thanks{bapoczos@cs.cmu.edu}}

\author[3]{
Larry Wasserman
\thanks{larry@stat.cmu.edu}}

\affil[1]{Computer Science Department\\
Carnegie Mellon University}
\affil[2]{Machine Learning Department\\
Carnegie Mellon University}
\affil[3]{Statistics Department\\
Carnegie Mellon University}

\maketitle

\input{abstract.tex}
\input{intro2.tex}
\input{related.tex}
\input{results2.tex}

\input{proofs.tex}

\input{experiments.tex}
\input{discussion.tex}

\section*{Acknowledgements}
This research is supported by DOE grant DESC0011114, NSF Grants DMS-0806009, IIS1247658, and IIS1250350, and Air Force Grant FA95500910373.
AK is supported in part by a NSF Graduate Research Fellowship.

\bibliography{divergence}
\bibliographystyle{plain}

\appendix
\input{app_upper.tex}

\input{app_lower.tex}
\input{app_aux.tex}

\end{document}

%% file: abstract.tex
\begin{abstract}
We consider nonparametric estimation of $L_2$, R\'{e}nyi-$\alpha$ and Tsallis-$\alpha$ divergences between continuous distributions.
Our approach is to construct estimators for particular integral functionals of two densities and translate them into divergence estimators.
For the integral functionals, our estimators are based on corrections of a preliminary plug-in estimator.
We show that these estimators achieve the parametric convergence rate of $n^{-1/2}$ when the densities' smoothness, $s$, are both at least $d/4$ where $d$ is the dimension.
We also derive minimax lower bounds for this problem which confirm that $s > d/4$ is necessary to achieve the $n^{-1/2}$ rate of convergence.
We validate our theoretical guarantees with a number of simulations.
\end{abstract}

%% file: intro2.tex
\section{Introduction}
\label{sec:intro}

Given samples from two distributions, one fundamental and classical question to ask is: how close are the two distributions?
First, one must specify what it means for two distributions to be close, for which a number of {\em divergences} have been proposed.
Then there is the statistical question: how does one estimate divergence given samples from two distributions. 
In this paper, we propose and analyze estimators for three common divergences. 

Divergence estimation has a number of applications across machine learning and statistics.
In statistics, one can use these estimators to construct two-sample and independence tests~\cite{pardo2005statistical}.
In machine learning, it is often convenient to view training data as a set of distributions and use divergences to estimate dissimilarity between examples.
This idea has been used in neuroscience, where the neural response pattern of an individual is modeled as a distribution, and divergence is used to compare responses across subjects~\cite{johnson2001information}.
It has also enjoyed success in computer vision, where features are computed for each patch of an image and these feature vectors are modeled as independent draws from an underlying distribution~\cite{poczos2012nonparametric}.

For these applications and others, it is crucial to accurately estimate divergences given samples drawn independently from each distribution. 
In the nonparametric setting, a number of authors have proposed various estimators which are provably consistent.
However, apart from a few examples, the actual {\em rates of convergence} of these estimators and the minimax optimal rates are still unknown.

In this work, we propose three estimators for the $L_2^2$, R\'{e}nyi-$\alpha$, and Tsallis-$\alpha$ divergence between two continuous distributions. 
Our strategy is to correct an initial plug-in estimator by estimates of the higher order terms in the von Mises expansion of the divergence functional. 
We establish the rates of convergence for these estimators under the assumption that both densities belong to a H\"{o}lder class of smoothness $s$.
Concretely, we show that the plug-in estimator achieves rate $n^{\frac{-s}{2s+d}}$ while correcting by the first order terms in the expansion results in an $n^{-\min\{\frac{2s}{2s+d}, 1/2\}}$-estimator and correcting further by the second order terms gives an $n^{-\min\{\frac{3s}{2s+d}, 1/2\}}$-estimator.
These last two estimators achieve the parametric $n^{-1/2}$ rate as long as the smoothness $s$ is larger than $d/2, d/4$, respectively, where $d$ is the dimension.
Moreover the first-order estimator, while worse statistically than the second-order estimator, is computationally very elegant.
These results contribute to our fairly limited knowledge on this important problems~\cite{nguyen2010estimating,singh2014generalized}.

We also address the issue of {\em statistical optimality} by deriving a minimax lower bound on the convergence rate. 
Specifically, we show that one cannot estimate these quantities at better than $n^{\frac{-4s}{4s+d}}$-rate when $s \le d/4$ and $n^{-1/2}$-rate otherwise.
This establishes the optimality of our best estimator in the smooth regime and also that $d/4$ is the critical smoothness for this problem. 

The remainder of this manuscript is organized as follows.
After discussing some related work on divergence estimation and the closely-related entropy estimation in Section~\ref{sec:related}, we present our estimators and main results in Sections~\ref{sec:estimators} and~\ref{sec:results}.
We provide proof sketches in Section~\ref{sec:proofs}. 
We present some numerical simulations in Section~\ref{sec:experiments} and conclude with some open questions in Section~\ref{sec:discussion}.
We defer many proof details and several calculations to the appendices. 

\subsection{Preliminaries}
Let us begin by standardizing notation and presenting some basic definitions. 
We will be concerned with two densities, $p,q: [0,1]^d \rightarrow \RR_{\ge 0}$ where $d$ denotes the dimension. 
Formally, letting $\mu$ denote the Lebesgue measure on $[0,1]^d$, we are interested in two probability distributions $\PP,\QQ$ with Radon-Nikodym derivatives $p = d\PP/d\mu, q = d\QQ/d\mu$.
Except for in this section, we will operate exclusively with the densities.
Throughout, the samples $\{X_i\}_{i=1}^n$ will be drawn independently from $p$ while the samples $\{Y_i\}_{i=1}^n$ will be drawn independently from $q$. 
For simplicity, assume that we are given $n$ samples from each distribution, although it is not hard to adjust the estimators and results to unequal sample sizes.
The divergences of interest are:
\begin{packed_enum}
\item \textbf{$L_2^2$-divergence}
\[
L_2^2(p,q) = \int (p(x)-q(x))^2 d\mu(x)
\]
\item \textbf{R\'{e}nyi-$\alpha$ Divergence}~\cite{renyi1961measures}
\[
D_\alpha(p, q) = \frac{1}{\alpha-1} \log\left( \int p^\alpha(x) q^{1-\alpha}(x) d\mu(x)\right)
\]
\item \textbf{Tsallis-$\alpha$ Divergence}~\cite{tsallis1988possible}
\[
T_\alpha(p,q) = \frac{1}{\alpha-1}\left( \int p^\alpha(x) q^{1-\alpha}(x)d\mu(x) - 1\right)
\]
\end{packed_enum}
Technically, these divergences are functionals on distributions, rather than densities, but we will abuse notation and write them as above.
As a unification, we consider estimating functionals of the form, $T(p,q) = \int p^\alpha(x)q^\beta(x)d\mu(x)$ for given $\alpha, \beta$.
Various settings of $\alpha,\beta$ yield the main terms in the divergences, and we will verify that estimators for $T(p,q)$ result in good divergence estimators.

The \emph{sine qua non} of our work is the \textbf{von Mises expansion}\footnote{See Chapter 20 of van der Vaart's book for an introduction to von Mises calculus~\yrcite{van2000asymptotic}.}.
Given a functional $T$ mapping distributions to the reals, the first-order von Mises expansion is:
\[
T(F) = T(G) + dT(G; F-G) + R_2, 
\]
where $F$ and $G$ are distributions, $R_2$ is a remainder term, and $dT(G; F-G)$ is the Gateaux derivative of $T$ at $G$ in the direction of $F-G$:
\[
dT(G; F-G) = \lim_{\tau \rightarrow 0} \frac{T(G + \tau(F-G)) - T(G)}{\tau}.
\]
In our work, $T$ is always of the form $T(F) = \int \phi(f)d\mu$ where $f = dF/d\mu$ is the Radon-Nikodym derivative and $\phi$ is differentiable.
In this case, the von Mises expansion reduces to a functional Taylor expansion on the densities\footnote{See Lemma~\ref{lem:gateaux} in the Appendix.}:
\ifthenelse{\equal{\version}{arxiv}}{
\begin{eqnarray*}
T(F) = T(G) +  \int \frac{\partial \phi(g(x))}{\partial g(x)} (f(x) - g(x))d\mu(x) +
O(\|f - g\|_2^2).
\end{eqnarray*}
}{
\begin{eqnarray*}
T(F) = T(G) +  \int \frac{\partial \phi(g(x))}{\partial g(x)} (f(x) - g(x))d\mu(x) +\\
O(\|f - g\|_2^2).
\end{eqnarray*}
}
We generalize these ideas to functionals of two distributions and with higher order expansions analogous to the Taylor expansion.
We often write $T(f)$ instead of $T(F)$.

%% file: related.tex
\section{Related Work}
\label{sec:related}

Divergence estimation and its applications have received considerable attention over the past several decades.
Pardo provides a fairly comprehensive discussion of methods and applications in the context of discrete distributions~\yrcite{pardo2005statistical}.

Only recently has attention shifted to the continuous, nonparametric setting, where a number of efforts have established consistent estimators.
Many of the approaches are based on nearest-neighbor graphs~\cite{hero1999estimation,wang2009divergence,poczos2011estimation,kallberg2012estimation}.
For example, P\'{o}czos and Schneider use a $k$-nearest-neighbor estimator and show that one does not need a consistent density estimator to consistently estimate R\'{e}nyi-$\alpha$ and Tsallis-$\alpha$ divergences.
A number of other authors have also proposed consistent estimators via the empirical CDF or histograms~\cite{wang2005divergence,perez2008kullback}.
Unfortunately, the rates of convergence for all of these methods are still unknown.

Singh and Poczos~\yrcite{singh2014generalized} recently established a rate of convergence for an estimator based on simply plugging kernel density estimates into the divergence functional.
Their estimator converges at $n^{\frac{-s}{s+d}}$-rate when $s < d$ and $n^{-1/2}$ otherwise which matches some existing results on estimating entropy functionals~\cite{liu2012exponential}.
In comparison, we show that corrections of the plug-in estimator lead to faster convergence rates and that the $n^{-1/2}$ rate can be achieved at the much lower smoothness of $s > d/4$. 
Moreover we establish a minimax lower bound for this problem, which shows that $d/4$ is the critical smoothness index.

Nguyen et al.~\yrcite{nguyen2010estimating} construct an estimator for Csisz\'{a}r $f$-divergences via regularized $M$-estimation and prove a rate of convergence when the likelihood-ratio $d\PP/d\QQ$ belongs to a Reproducing Kernel Hilbert Space.
Their rate depends on the complexity of this RKHS, but it is not clear how to translate these assumptions into our H\"{o}lderian one, so the results are somewhat incomparable.


K\"{a}llberg and Seleznjev~\yrcite{kallberg2012estimation} study an $\epsilon$-nearest neighbor estimator for the $L_2^2$-divergence that enjoys the same rate of convergence as our projection-based estimator. 
They prove that the estimator is asymptotically normal in the $s > d/4$ regime, which one can also show for our estimator.
In the more general setting of estimating polynomial functionals of the densities, they only show consistency of their estimator, while we also characterize the convegence rate.

A related and flourishing line of work is on estimating entropy functionals.
The majority of the methods are graph-based, involving either nearest neighbor graphs or spanning trees over the data~\cite{hero2002convergence,leonenko2008class,leonenko2010statistical,pal2010renyi,sricharan2010empirical}.
One exception is the KDE-based estimator for mutual information and joint entropy of Liu, Lafferty, and Wasserman~\yrcite{liu2012exponential}.
A number of these estimators come with provable convergence rates. 

While it is not clear how to port these ideas to divergence estimation, it is still worth comparing rates.
The estimator of Liu et al.~\yrcite{liu2012exponential} converges at rate $n^{\frac{-s}{s+d}}$, achieving the parametric rate when $s > d$.
Similarly, Sricharan et al.~\yrcite{sricharan2010empirical} show that when $s > d$ a $k$-NN style estimator achieves rate $n^{-2/d}$ (in absolute error) ignoring logarithmic factors. 
In a follow up work, the authors improve this result to $O(n^{-1/2})$ using an ensemble of weak estimators, but they require $s > d$ orders of smoothness~\cite{sricharan2012ensemble}.
In contrast, our estimators achieve the parametric $n^{-1/2}$ rate at lower smoothness ($s > d/2, d/4$ for the first-order and second-order estimators, respectively) and enjoy a faster rate of convergence uniformly over smoothness.

Interestingly, while many of these methods are plug-in-based, the choice of tuning parameter typically is sub-optimal for density estimation.
This contrasts with our technique of correcting optimal density estimators.

We are not aware of any lower bounds for divergence estimation, although analogous results have been established for the entropy estimation problem.
Specifically, Birg\'{e} and Massart~\yrcite{birge1995estimation} prove a $n^{\frac{-4s}{4s+d}}$-lower bound for estimating integral functionals of a density.
Hero et al.~\yrcite{hero2002convergence} give a matching lower bound for estimating R\'{e}nyi-$\alpha$ entropies. 

Finally, our estimators and proof techniques are based on several classical works on estimating integral functionals of a density. 
The goal here is to estimate $\int \phi(f(x))d\mu(x)$, for some known function $\phi$, given samples from $f$. 
A series of papers show that $n^{-1/2}$ rate of convergence is attainable if and only if $s > d/4$, which is analogous to our results~\cite{birge1995estimation,laurent1996efficient,kerkyacharian1996estimating,bickel1988estimating}.
Of course, our results pertain to the two-density setting, which encompasses the divergences of interest.
We also generalize some of these results to the multi-dimensional setting.




%% file: results2.tex

\section{The Estimators}
\label{sec:estimators}
Recall that we are interested in estimating integral functionals of the form $T(p,q) = \int p^\alpha(x)q^\beta(x)$. 
As an initial attempt, with estimators $\phat$ and $\qhat$ for $p$ and $q$, we can use the plug-in estimator $\Thatplugin = T(\phat, \qhat)$. 
Via the von Mises expansion of $T(p,q)$, the error is of the form: 
\[
|\Thatplugin - T(p,q)| \le c_1 \|\phat-p\|_1 + c_2 \|\qhat-q\|_1.
\]
Classical results on density estimation then suggest that $\Thatplugin$ will enjoy a $n^{\frac{-s}{2s+d}}$-rate~\cite{devroye1985nonparametric}.

A better convergence rate can be achieved by correcting the plug-in estimator with estimates of the linear term in the von Mises expansion. 
Informally speaking, the remainder of the first order expansion is $O(\|\phat-p\|_2^2 + \|\qhat-q\|_2^2)$ which decays with $n^{\frac{-2s}{2s+d}}$, while the linear terms can be estimated at $n^{-1/2}$-rate.
This estimator, which we call $\Thatlin$ enjoys a faster convergence rate than $\Thatplugin$.

It is even better to augment the plug-in estimator with both the first and second-order terms of the expansion.
Here the remainder decays at rate $n^{-\frac{3s}{2s+d}}$ while the linear and quadratic terms can be estimated at $n^{-1/2}$ and $n^{\frac{-4s}{4s+d}}$ rate respectively.
This corrected estimator $\Thatquad$ achieves the parametric rate whenever the smoothness $s > d/4$ which we will show to be minimax optimal.

We now formalize these heuristic developments\footnote{See Appendices~\ref{app:von} and~\ref{app:est} for details omitted in this section.}. 
Below we enumerate the terms in the first and second order von Mises expansions that we will estimate or compute: 
\begin{align*}
\theta_{1,1}^p & = \EE_{X \sim p} \alpha \phat^{\alpha-1}(X)\qhat^\beta(X)\\
\theta_{1,1}^q & = \EE_{Y \sim q} \beta \phat^{\alpha}(Y)\qhat^{\beta-1}(Y)\\
\theta_{2,1}^p & = \EE_{X\sim p} \alpha(2-\alpha-\beta) \phat^{\alpha-1}(X)\qhat^\beta(X)\\
\theta_{2,1}^q & = \EE_{Y\sim q} \beta(2-\alpha-\beta) \phat^{\alpha}(Y)\qhat^{\beta-1}(Y)\\
\theta_{2,2}^p & = \frac{1}{2} \int \alpha(\alpha-1)\phat^{\alpha-2}\qhat^\beta p^2\\
\theta_{2,2}^q & = \frac{1}{2} \int \beta(\beta-1)\phat^{\alpha}\qhat^{\beta-2} q^2\\
\theta_{2,2}^{p,q} & =  \int \alpha\beta \phat^{\alpha-1}\qhat^{\beta-1} pq\\
C_1 & = 1 - \alpha - \beta\\
C_2 & = 1 - \frac{3}{2}(\alpha+\beta) + \frac{1}{2}(\alpha+\beta)^2
\end{align*}
These definitions allow us to succinctly write the expansions of $T(p,q)$ about $T(\phat, \qhat)$:
\begin{eqnarray*}
&& T_0(p,q) = T(\phat,\qhat) + R_1\\
&& T_1(p,q) = C_1 T(\phat,\qhat) + \theta_{1,1}^p + \theta_{1,1}^q + R_2\\
&& T_2(p,q) = C_2 T(\phat,\qhat) + \sum_{i=1,2 \atop f = p,q}\theta_{2,i}^f + \theta_{2,2}^{p,q} + R_3,
\end{eqnarray*}
with remainders, $R_a = O(\|p-\phat\|_a^a + \|q-\qhat\|_a^a)$.


We now turn to estimation of the $\theta_{(\cdot), (\cdot)}^{(\cdot)}$ terms.
All of the $\theta_{(\cdot),1}^{(\cdot)}$ terms are \emph{linear}; that is, they are of the form $\theta = \EE_{Z\sim f}[\psi(Z)]$ where $\psi$ is known.
A natural estimator, given data $Z_1^n \sim f$, is the sample mean:
\begin{align*}
\hat{\theta} = \frac{1}{n}\sum_{j=1}^n \psi(Z_j).
\end{align*}

The terms $\theta_{(\cdot),2}^{(\cdot)}$ are of the form:
\[
\int \psi(x) f^2(x), \qquad \textrm{or} \qquad \int \psi(x) f(x) g(x),
\]
again with known $\psi$.
To estimate these terms, we have samples $X_1^n \sim f, Y_1^n \sim g$. 
If $\{\phi_k\}_{k \in D}$ is an orthonormal basis for $L_2([0,1]^d)$ then the estimator for the bilinear term is:
\begin{eqnarray}
\hat{\theta} = \frac{1}{n}\sum_{j=1}^n \sum_{k \in M} \left(\frac{1}{n}\sum_{i=1}^n \phi_k(X_i)\right) \phi_k(Y_j)\psi(Y_j),
\label{eq:both_estimator}
\end{eqnarray}
where $M \subset D$ is chosen to tradeoff the bias and the variance. 
To develop some intuition, if we knew $f$, we would simply use the sample mean $\frac{1}{n} \sum_{j=1}^n f(Y_j)\psi(Y_j)$.
Since $f$ is actually unknown, we replace it with an estimator formed by truncating its Fourier expansion.
Specifically, we replace $f$ with $\hat{f}(\cdot) = \sum_{k \in M} \hat{a}_k \phi_k(\cdot)$ with $\hat{a}_k = \frac{1}{n}\sum_{i=1}^n \phi_k(X_i)$.

For the quadratic functional, a projection estimator was proposed and analyzed by Laurent \yrcite{laurent1996efficient}:
\ifthenelse{\equal{\version}{arxiv}}{
\begin{equation}
\hat{\theta} = \frac{2}{n(n-1)} \sum_{k \in M} \sum_{i \ne j} \phi_k(X_i)\phi_k(X_j)\psi(X_j)
- \frac{1}{n(n-1)} \sum_{k,k' \in M} \sum_{i \ne j} \phi_k(X_i)\phi_{k'}(X_j)b_{k,k'}(\psi),
\label{eq:quadratic_estimator}
\end{equation}
}{
\begin{equation}
\begin{aligned}
\hat{\theta} = \frac{2}{n(n-1)} &\sum_{k \in M} \sum_{i \ne j} \phi_k(X_i)\phi_k(X_j)\psi(X_j)\\ 
- \frac{1}{n(n-1)} &\sum_{k,k' \in M} \sum_{i \ne j} \phi_k(X_i)\phi_{k'}(X_j)b_{k,k'}(\psi),
\end{aligned}
\label{eq:quadratic_estimator}
\end{equation}
}
where $b_{k,k'}(\psi) = \int \phi_k(x) \phi_{k'}(x) \psi(x)dx$. 
The first term in the estimator is motivated by the same line of reasoning as in the bilinear estimator while the second term significantly reduces the bias without impacting the variance.




Our final estimators for $T(p,q)$ are:
\begin{eqnarray*}
\Thatplugin &=& T(\phat,\qhat)\\
\Thatlin &=& C_1 T(\phat,\qhat) + \hat{\theta}_{1,1}^p + \hat{\theta}_{1,1}^q \\
\Thatquad &=& C_2 T(\phat,\qhat) + \sum_{i=1,2 \atop f = p,q}\hat{\theta}_{2,i}^f + \hat{\theta}_{2,2}^{p,q}.
\end{eqnarray*}

Before proceeding to our theoretical analysis, we mention some algorithmic considerations.
We estimate $\phat, \qhat$ with kernel density estimators, which, except for in $\Thatplugin$, we only train on half of the sample.
This gives us independent samples to estimate the $\hat{\theta}_{(\cdot), (\cdot)}^{(\cdot)}$ terms. 
Second, in our analysis, we will require that the KDEs are bounded above and below. 
Under the assumption that $p$ and $q$ are bounded above and below, we will show that clipping the original KDE will not affect the convergence rate.

Another important issue with density estimation over bounded domains, that applies to our setting, is that the standard KDE suffers high bias near the boundary.
To correct this bias, we adopt the strategy used by Liu et al.~\yrcite{liu2012exponential} of ``mirroring'' the data set over the boundaries.
We do not dwell too much on this issue, noting that this technique can be shown to suitably correct for boundary bias without substantially increasing the variance.
This augmented estimator can be shown to match the rates of convergence in the literature~\cite{devroye1985nonparametric, tsybakov2009introduction}.

Lastly, the estimators all require integration of the term $T(\phat,\qhat)$, which can be computationally burdensome, particularly in high dimension.
However, whenever $\alpha +\beta= 1$, as in the R\'{e}nyi-$\alpha$ and Tsallis-$\alpha$ divergences, the constants $C_1, C_2$ are zero, so the first term may be omitted. 
In this case $\Thatlin$ is remarkably simple; it involves training KDEs and estimating a specific linear functional of them via the sample mean.
Although this estimator is not minimax optimal, it enjoys a fairly fast rate of convergence while being computationally practical.
Unfortunately, even when $C_2=0$, the quadratic estimator still involves integration of the $b_{i,i'}$ terms.
We therefore advocate for $\Thatlin$ over $\Thatquad$ in practice, as $\Thatlin$ exhibits a better tradeoff between computational and statistical efficiency. 

\section{Theoretical Results}
\label{sec:results}
\begin{figure*}
\begin{center}
\includegraphics[width=\textwidth]{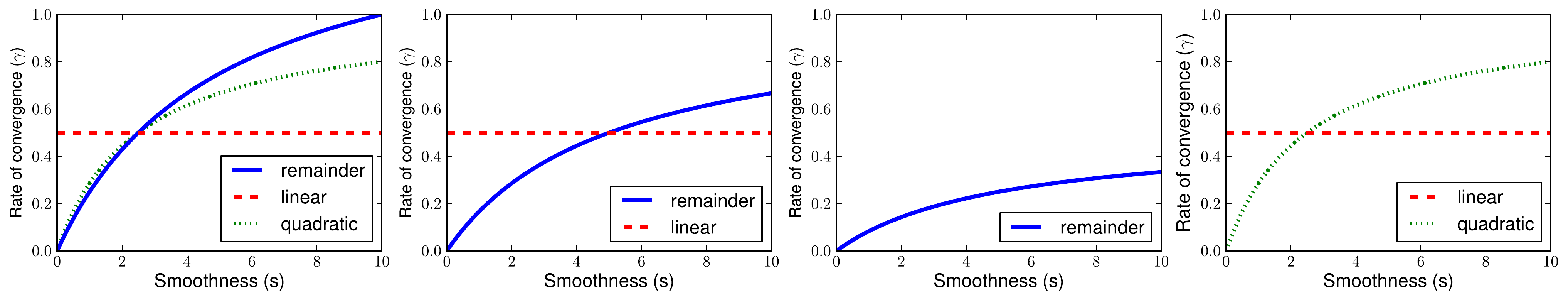}
\end{center}
\caption{Rates of convergence of the estimators $\Thatquad, \Thatlin, \Thatplugin$ along with the rate of convergence in the lower bound (Theorem~\ref{thm:lower_bound}).
Plot is $\gamma$ vs. smoothness $s$ with $d=10$, where the rate of convergence is $O(n^{-\gamma})$.
The rate of convergence for each estimator is the smallest of the rates of all terms in the von Mises expansion, which translates to the value of the lowest curves in the figure. }
\label{fig:rate_plot}
\end{figure*}

For our theoretical analysis, we will assume that the densities $p,q$ belong to $\Sigma(s,L)$, the \textbf{periodic H\"{o}lder class} of smoothness $s$, defined as follows:
\begin{definition}
For any tuple $r = (r_1, \ldots, r_d)$ define $D^r = \frac{\partial^{r_1+\ldots+r_d}}{\partial x_1^{r_1}\ldots\partial x_d^{r_d}}$.
The \textbf{periodic H\"{o}lder class} $\Sigma(s,L)$ is the subset of $L_2([0,1]^d)$ where for each $f\in \Sigma(s,L)$, the $r$th derivative is periodic for any tuple $r$ with $\sum_j r_j < s$ and:
\begin{eqnarray}
|D^r f(x) - D^r f(y)| \le L \|x-y\|^{s-|r|},
\end{eqnarray}
for all $x,y$ and for all tuples $r$ with $\sum_j r_j = \lfloor s \rfloor$ the largest integer strictly smaller than $s$. 
\end{definition}

We are now ready to state our main assumptions:
\begin{assum}[Smoothness]
$p,q \in \Sigma(s,L)$ for some known smoothness $s$. 
\label{ass:smoothness}
\end{assum}
\begin{assum}[Boundedness]
The densities are bounded above and below by known parameters $\kappa_l, \kappa_u$. 
Formally $0 < \kappa_l \le  p(x),q(x) \le \kappa_u < \infty$ for all $x \in [0,1]^d$.
\label{ass:bound}
\end{assum}
\begin{assum}[Kernel Properties]
The kernel $K \in \RR^d \rightarrow \RR$ satisfies:
\begin{align*}
(i) &\ \textrm{supp}(K) \in (-1, 1)^d\\
(ii) &\int K(x)d\mu(x) = 1\\
(iii) &\int \prod_{i=1}^d x_i^{r_i}K(x)d\mu(x) = 0, \forall r \in \NN^d : \sum_{i} r_i \le \lfloor s \rfloor
\end{align*}
\label{ass:kernel}
\end{assum}
\begin{assum}[Parameter Selection]
Set the KDE bandwidth $h \asymp n^{\frac{-1}{2s+d}}$.
For any projection-style estimator, set the number of basis elements $m \asymp n^{\frac{2d}{4s+d}}$.
\label{ass:hyper}
\end{assum}
The H\"{o}lderian assumption is standard in the nonparametric literature while the periodic assumption subsumes more standard boundary smoothness conditions~\cite{liu2012exponential}.
It is fairly straightforward to construct kernels meeting Assumption~\ref{ass:kernel}~\cite{tsybakov2009introduction}, while the boundedness assumption is common in the literature on estimating integral functionals of a density~\cite{birge1995estimation}.

The following theorem characterizes the rate of convergence of our estimators $\Thatplugin, \Thatlin, \Thatquad$:
\begin{theorem}
Under Assumptions~\ref{ass:smoothness}-~\ref{ass:hyper} we have:
\begin{eqnarray}
\EE\left[|\Thatplugin - T(p,q)| \right] &=& O\left(n^{\frac{-s}{2s+d}}\right)\\
\EE\left[|\Thatlin - T(p,q)| \right] &=& O\left(n^{-1/2} + n^{\frac{-2s}{2s+d}}\right)\\
\EE\left[|\Thatquad - T(p,q)| \right] &=& O\left(n^{-1/2} + n^{\frac{-3s}{2s+d}}\right). 
\end{eqnarray}
All expectations are taken with respect to $X_1^n, Y_1^n$.
When $s = d/4$, $\Thatquad$ enjoys $O(n^{-1/2+\epsilon})$ rate of convergence for any $\epsilon > 0$\footnote{The constant is exponential in $\epsilon$ and is infinite for $\epsilon = 0$.}.
$\Thatlin$ and $\Thatquad$ achieve the parametric rate when $s > d/2, d/4$ respectively.
\label{thm:upper_bound}
\end{theorem}

Before commenting on the upper bound and presenting some consequences, we address the question of \emph{statistical efficiency}.
Clearly $\Thatplugin$ and $\Thatlin$ are not rate-optimal, since $\Thatquad$ achieves a faster rate of convergence, but is $\Thatquad$ minimax optimal?
We make some progress in this direction with a minimax lower bound on the rate of convergence.

\begin{theorem}
\label{thm:lower_bound}
Under Assumptions~\ref{ass:smoothness} and~\ref{ass:bound}, as long as both $\alpha,\beta \ne 0, 1$, then with $\gamma_\star = \min\{4s/(4s+d), 1/2\}$ and for any $\epsilon > 0$:
\begin{align*}
\liminf_{n \rightarrow \infty} \inf_{\wT_n} \sup_{p,q\in \Sigma(s,L)} \PP_{p,q}^{n}\left[ |\wT_n - T| \ge \epsilon n^{-\gamma_\star} \right] \ge c> 0.
\end{align*}
\end{theorem}

For a pictorial understanding of the rates of convergence and the lower bound, we plot the exponent $\gamma$ for each of the terms in the von Mises expansion as a function of the smoothness $s$ in Figure~\ref{fig:rate_plot}.
The estimator $\Thatquad$ has three terms, with rates $n^{-1/2}, n^{\frac{-4s}{4s+d}}$, and $n^{\frac{-3s}{2s+d}}$ respectively which achieves the parametric rate $n^{-1/2}$ when $s > d/4$ and is $n^{\frac{-3s}{2s+d}}$ in the low-smoothness regime. 
The linear estimator only achieves the parametric rate while $s > d/2$ while $\Thatplugin$ only approaches the parametric rate as $s \rightarrow \infty$. 
Consequently these estimators are statistically inferior to $\Thatquad$.
In the last plot we show a lower bound on the rate of convergence from Theorem~\ref{thm:lower_bound}, which is $n^{\frac{-4s}{4s+d}}$ when $s \le d/4$ and $n^{-1/2}$ when $s > d/4$. 

The lower bound rate deviates slightly from the upper bound for $\Thatquad$ in the low-smoothness regime, showing that $\Thatquad$ is also not minimax-optimal uniformly over $s$. 
This sub-optimality appears even when estimating integral functionals of a single density~\cite{birge1995estimation}.
In that context, achieving the optimal rate of convergence in the non-smooth regime involves further correction by the third order term in the expansion~\cite{kerkyacharian1996estimating}.
It seems as if the same ideas can be adapted to the two-density setting, although we believe computational considerations would render these estimators impractical.

In the smooth regime ($s > d/4$) we see that the parametric $n^{-1/2}$ rate is both necessary and sufficient. 
This critical smoothness index of $s = d/4$ was also observed in the context of estimating integral functionals of densities~\cite{birge1995estimation,laurent1996efficient}.

When $s = d/4$, the quadratic estimator achieves $n^{-1/2+\epsilon}$ rate for any $\epsilon > 0$, where the constant is exponential in $\epsilon$, and thus deviates slightly from the lower bound. 
This phenomenon arises from using the projection-based estimators for the quadratic term.
Establishing the rate of convergence for these estimators requires working in a Sobolev space rather than the H\"{o}lder class.
In translating back to the H\"{o}lderian assumption, we lose a small factor in the smoothness, since the Sobolev space only contains the H\"{o}lder space if the former is less smooth than the latter. 

The lower bound on estimating integral functionals in Theorem~\ref{thm:lower_bound} almost immediately implies a lower bound for Tsallis-$\alpha$ divergences.
For R\'{e}nyi-$\alpha$, some care must be taken in the translation, but we are able to prove the same lower bound as long as $D_\alpha(p,q)$ is bounded. 
The idea behind these extensions is to translate an estimator $\hat{D}$ for the divergence into an estimator $\hat{T}$ for $T(p,q)$.
We then argue that if $\hat{D}$ enjoyed a fast rate of convergence, so would $\hat{T}$, which leads to a contradiction of the theorem. 
Unfortunately, Theorem~\ref{thm:lower_bound} does not imply a lower bound for $L_2^2$ divergence, since we are unable to handle the $\alpha=\beta=1$ case, which is exactly the cross term in the $L_2^2$-divergence.

Our proof requires that both $\alpha, \beta$ are both not $0$ or $1$, which is not entirely surprising.
If $\alpha=\beta=0$, $T(p,q)$ is identically zero, so one should not be able to prove a lower bound.
Similarly $\alpha=0, \beta=1$ or vice versa, $T(p,q) = 1$ for any $p,q$, so we have efficient, trivial estimators.

The only non-trivial case is $\alpha=\beta=1$ and we conjecture that the $n^{-\gamma_\star}$ rate is minimax optimal there, although our proof does not apply. 
Our proof strategy involves fixing $q$ and perturbing $p$, or vice versa.
In this approach, one can view the optimal estimator as having knowledge of $q$, so if $\alpha=1$, the sample average is a $n^{-1/2}$-consistent estimator, which prevents us from achieving the $n^{-\gamma_\star}$ rate.
We believe this is an artifact of our proof, and by perturbing both $p$ and $q$ simultaneously, we conjecture that one can prove a minimax lower bound of $n^{-\gamma_\star}$ when $\alpha=\beta=1$.


\subsection{Some examples}
We now show how an estimate of $T(p,q)$ can be used to estimate the divergences mentioned above. 
Plugging $\Thatquad$ into the definition of R\'{e}nyi-$\alpha$ and Tsallis-$\alpha$ divergences, we immediately have the following corollary:
\begin{corollary}[Estimating R\'{e}nyi-$\alpha$, Tsallis-$\alpha$ divergences]
\label{cor:renyi}
Under Assumptions~\ref{ass:smoothness}-~\ref{ass:hyper}, as long as $D_\alpha(p,q) \ge c > 0$ for some constant $c$, the estimators:
\begin{align*}
\hat{D}_\alpha &= \frac{1}{\alpha-1}\log(\Thatquad), \qquad 
\hat{T}_\alpha = \frac{1}{\alpha-1}(\Thatquad - 1),
\end{align*}
both with $\beta=1-\alpha$, satisfy:
\begin{align}
\EE_{X_1^n, Y_1^n} |\hat{D}_\alpha - D_\alpha(p,q)| \le c\left(n^{-1/2} + n^{\frac{-3s}{2s+d}}\right)\\
\EE_{X_1^n, Y_1^n} |\hat{T}_\alpha - T_\alpha(p,q)| \le c\left(n^{-1/2} + n^{\frac{-3s}{2s+d}}\right)
\end{align}
\end{corollary}

As we mentioned before, when $\beta = 1-\alpha$, for both the linear and quadratic estimators, one can omit the term $T(\phat, \qhat)$ as the constants $C_1,C_2 = 0$.
However, $\Thatquad$ is still somewhat impractical due to the numeric integration in the quadratic terms.
On the other hand, the linear estimator $\Thatlin$ is computationally very simple, although its convergence rate is $O(n^{-1/2} + n^{\frac{-2s}{2s+d}})$.


For the $L_2^2$ divergence, instead of applying Theorem~\ref{thm:upper_bound} directly, it is better to directly use the quadratic and bilinear estimators for the terms in the factorization. 
Specifically, let $\theta_p = \int p^2$ and define $\hat{\theta}_p$ by Equation~\ref{eq:quadratic_estimator} with $\psi(x) = 1$.
Define $\theta_q, \hat{\theta}_q$ analogously and finally define $\theta_{p,q} = 2 \int pq$ with $\hat{\theta}_{p,q}$ given by Equation~\ref{eq:both_estimator} where $\psi(x) = 2$. 
As a corollary of Theorem~\ref{thm:quadratic_rate} below, we have:
\begin{corollary}[Estimating $L_2^2$-divergence]
\label{cor:l2}
Under Assumptions~\ref{ass:smoothness}-~\ref{ass:hyper}, the estimator $\hat{L} = \hat{\theta}_p + \hat{\theta}_q - \hat{\theta}_{p,q}$ for $L_2^2(p,q)$ satisfies:
\begin{eqnarray}
\EE_{X_1^n, Y_1^n}\left[|\hat{L}- L_2^2(p,q)|\right] = O(n^{-1/2} + n^{\frac{-4s}{4s+d}}).
\end{eqnarray}
\end{corollary}

Notice that for both quadratic terms, the $b_{i,i'}$ terms in Equation~\ref{eq:quadratic_estimator} are $\mathbf{1}[i = i']$ since $\psi(x) = 1$ and since $\{\phi_k\}$ is an orthonormal collection.
Thus the estimator $\hat{L}$ is computationally attractive, as numeric integration is unnecessary. 
In addition, we do not need KDEs, removing the need for bandwidth selection, although we still must select the basis functions used in the projection.


%% file: proofs.tex
\section{Proof Sketches}
\label{sec:proofs}
\subsection{Upper Bound}
The rates of convergence for $\Thatplugin, \Thatlin$, and $\Thatquad$ come from analyzing the kernel density estimators and the estimators for $\hat{\theta}_{(\cdot), (\cdot)}^{(\cdot)}$. 
Recall that we must use truncated KDEs $\phat, \qhat$ with boundary correction, so standard analysis does not immediately apply.
However, we do have the following theorem establishing that truncation does not affect the rate, which generalizes previous results to high dimension~\cite{birge1995estimation}.
\begin{theorem}
\label{thm:kde_rate}
Let $f$ be a density satisfying Assumptions~\ref{ass:smoothness}-~\ref{ass:hyper} and suppose we have $X_1^n \sim f$.
The truncated KDE $\hat{f}_n$ satisfies:
\begin{align*}
\EE_{X_1^n} \|\hat{f}_n - f\|_p^p \le C n^{\frac{-ps}{2s+d}}.
\end{align*}
\end{theorem}


It is simple exercise to show that the linear terms can be estimated at $n^{-1/2}$ rate. 
As for the quadratic terms $\theta_{2,2}^p, \theta_{2,2}^q$, and $\theta_{2,2}^{p,q}$, we let $D$ index the multi-dimensional Fourier basis where each function $\phi_k(x) = e^{2\pi i k^Tx}$ is indexed by a $d$-dimensional integral vector (i.e. $k \in \ZZ^d$). 
We have:
\begin{theorem}
\label{thm:quadratic_rate}
Let $f,g$ be densities in $\Sigma(s,L)$ and let $\psi$ be a known bounded function. 
Let $\phi_k$ be the Fourier basis and $M$ the set of basis elements with frequency not exceeding $m_0^{1/d}$, where $m_0 \asymp n^{\frac{2d}{4s'+d}}$ for some $s' < s$.
If $\theta = \int \psi(x) f(x)g(x)$ and $\hat{\theta}$ is given by Equation~\ref{eq:both_estimator} or if $\theta = \int \psi(x)f^2(x)$ and $\hat{\theta}$ is given by Equation~\ref{eq:quadratic_estimator},
then:
\begin{eqnarray}
\EE[(\hat{\theta}-\theta)^2] \le O\left(n^{-1} + n^{\frac{-8s'}{4s'+d}}\right).
\end{eqnarray}
\end{theorem}

Theorem~\ref{thm:upper_bound} follows from these results, the von Mises expansion, and the triangle inequality.



\subsection{Lower Bound}

The first part of the lower bound is an application of Le Cam's method and generalizes a proof of Birge and Massart~\yrcite{birge1995estimation}.
We begin by reducing the estimation problem to a simple-vs.-simple hypothesis testing problem.
We will use the squared Hellinger distance, defined as:
\begin{align*}
h^2(p,q) = \int \left(\sqrt{p(x)} - \sqrt{q(x)}\right)^2d \mu(x)
\end{align*}
\begin{lemma}
Let $T$ be a functional defined on some subset of a parameter space $\Theta \times \Theta$ which contains $(p,q)$ and $(g_\lambda, q) \forall \lambda$ in some index set $\Lambda$. 
Define $\bar{G}^n = \frac{1}{|\Lambda|} \sum_{\lambda \in \Lambda} G^n_\lambda$ where $G_\lambda$ has density $g_\lambda$. 
If:
\begin{align*}
(i)\ & h^2(P^n \times Q^n, \bar{G}^n \times Q^n) \le \gamma < 2\\
(ii)\ & T(p,q) \ge 2\beta + T(g_\lambda, q) \ \forall \lambda \in \Lambda
\end{align*}
Then:
\[
\inf_{\hat{T}_{n}} \sup_{p \in \Theta} \PP_{p,q}^{n}\left[ |\hat{T}_{n} - T(p,q)| > \beta\right] \ge c_\gamma,
\]
where $c_\gamma = \frac{1}{2} [ 1 - \sqrt{\gamma(1-\gamma/4)}]$.
\label{lem:lecam}
\end{lemma}

To construct the $g_\lambda$ functions, we partition the space $[0,1]^d$ into $m$ cubes $R_j$ and construct functions $u_j$ that are compactly supported on $R_j$.
We then set $g_\lambda = p+K \sum_{j=1}^m \lambda_j u_j$ for $\lambda \in \Lambda = \{-1,1\}^m$.
By appropriately selecting the functions $u_j$, we can ensure that:
\begin{eqnarray*}
g_\lambda &\in& \Sigma(s,L),\\
T(p,q) - T(g_\lambda, q) &\ge& \Omega(K^2)\\
h^2(P^n\times Q^m, \bar{G}^n \times Q^m) &\le& O(n^2 K^4/m).
\end{eqnarray*}
Ensuring smoothness requires $K = O(m^{-s/d})$ at which point, making the Hellinger distance $O(1)$ requires $m = \Omega(n^{\frac{2d}{4s+d}})$. 
With these choices we can apply Lemma~\ref{lem:lecam} and arrive at the lower bound since $K^2 = m^{-2s/d} = n^{\frac{-4s}{4s+d}}$. 

As for the second part of the theorem, the $n^{-1/2}$ lower bound, we use a (to our knowledge) novel proof technique which we believe may be applicable in other settings.
The first ingredient of our proof is a lower bound showing that one cannot estimate a wide class of quadratic functionals at better than $n^{-1/2}$ rate.
We provide a proof of this result based on Le Cam's method in the appendix although related results appear in the literature~\cite{donoho1991geometrizing}.
Then starting with the premise that there exists an estimator $\hat{T}$ for $T(p,q)$ with rate $n^{-1/2-\epsilon}$, we construct an estimator for a particular quadratic functional with $n^{-1/2-\epsilon}$ convergence rate, and thus arrive at a contradiction.
A somewhat surprising facet of this proof technique is that the proof has the flavor of an upper bound proof; in particular, we apply Theorem~\ref{thm:kde_rate} in this argument.

The proof works as follows:
Suppose there exists a $\hat{T}_n$ such that $|\hat{T}_n - T(p,q)| \le c_1 n^{-1/2-\epsilon}$ for all $n$. 
If we are given $2n$ samples, we can use the first half to train KDEs $\phat_n,\qhat_n$, and the second half to compute $\hat{T}_n$. 
Armed with these quantities, we can build an estimator for the first and second order terms in the von Mises expansion, which, once $\phat_n,\qhat_n$ are fixed, is simply a quadratic functional of the densities.
The precise estimator is $\hat{T}_n - C_2T(\phat_n, \qhat_n)$.
The triangle inequality along with Theorem~\ref{thm:kde_rate} shows that this estimator converges at rate $n^{-1/2-\epsilon} +n^{\frac{-3s}{2s+d}}$ which is $o(n^{-1/2})$ as soon as $s > d/4$. 
This contradicts the minimax lower bound for estimating quadratic functionals of H\"{o}lder smooth densities.
We refer the interested reader to the appendix for details of the proof.

%% file: experiments.tex
\section{Experiments}
\label{sec:experiments}

\ifthenelse{\equal{\version}{arxiv}}{
\begin{figure}[h]
\begin{center}
\includegraphics[scale=0.3]{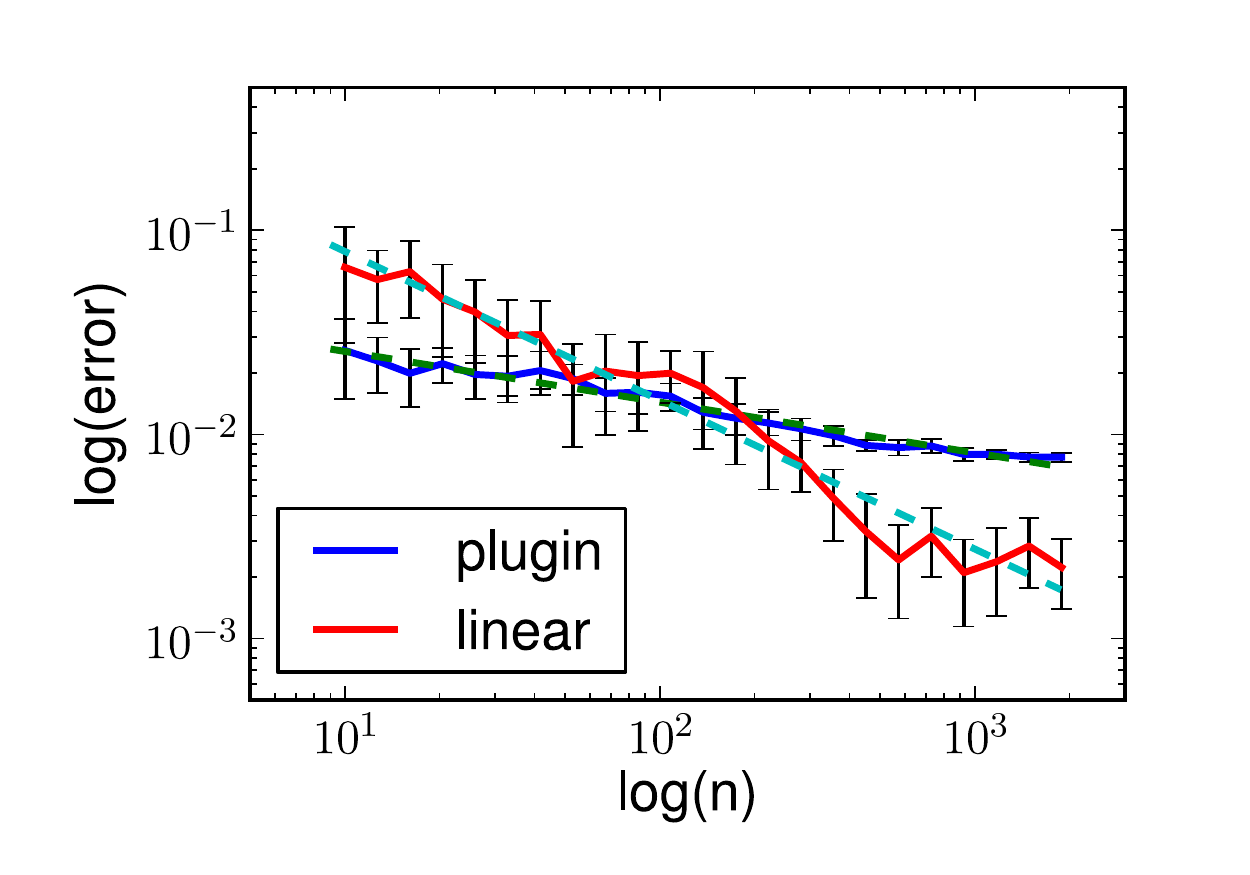} \hspace{-0.4cm}
\includegraphics[scale=0.3]{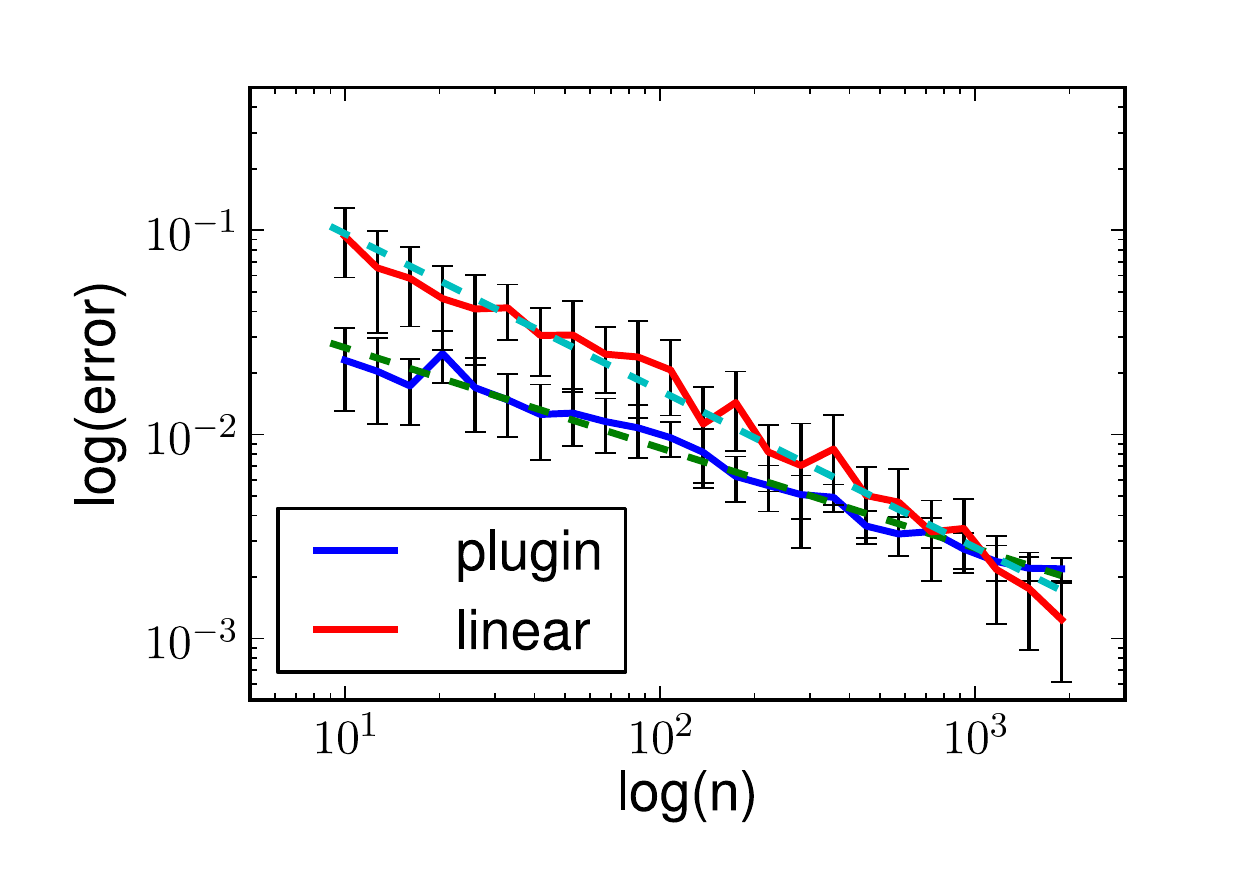} \hspace{-0.4cm}
\includegraphics[scale=0.3]{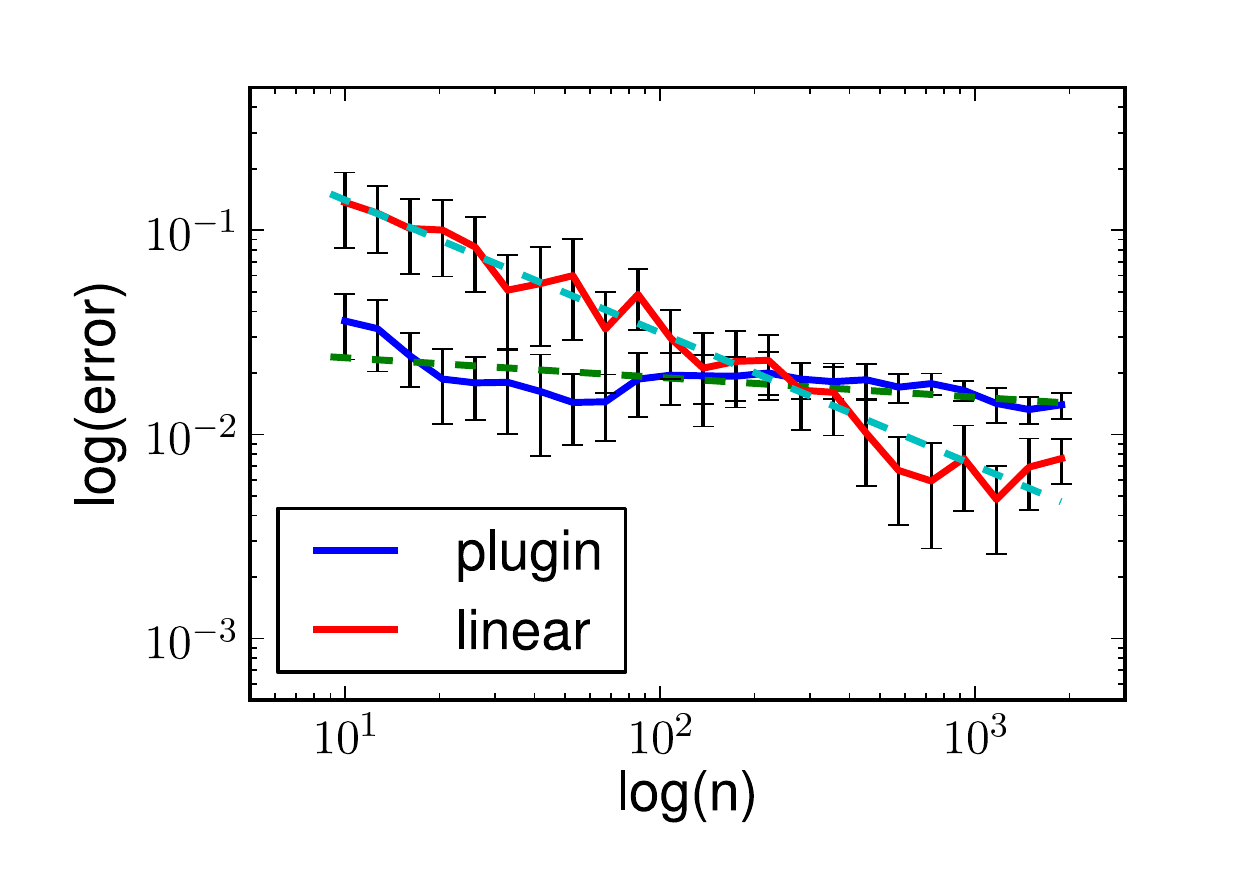} \hspace{-0.4cm}
\includegraphics[scale=0.3]{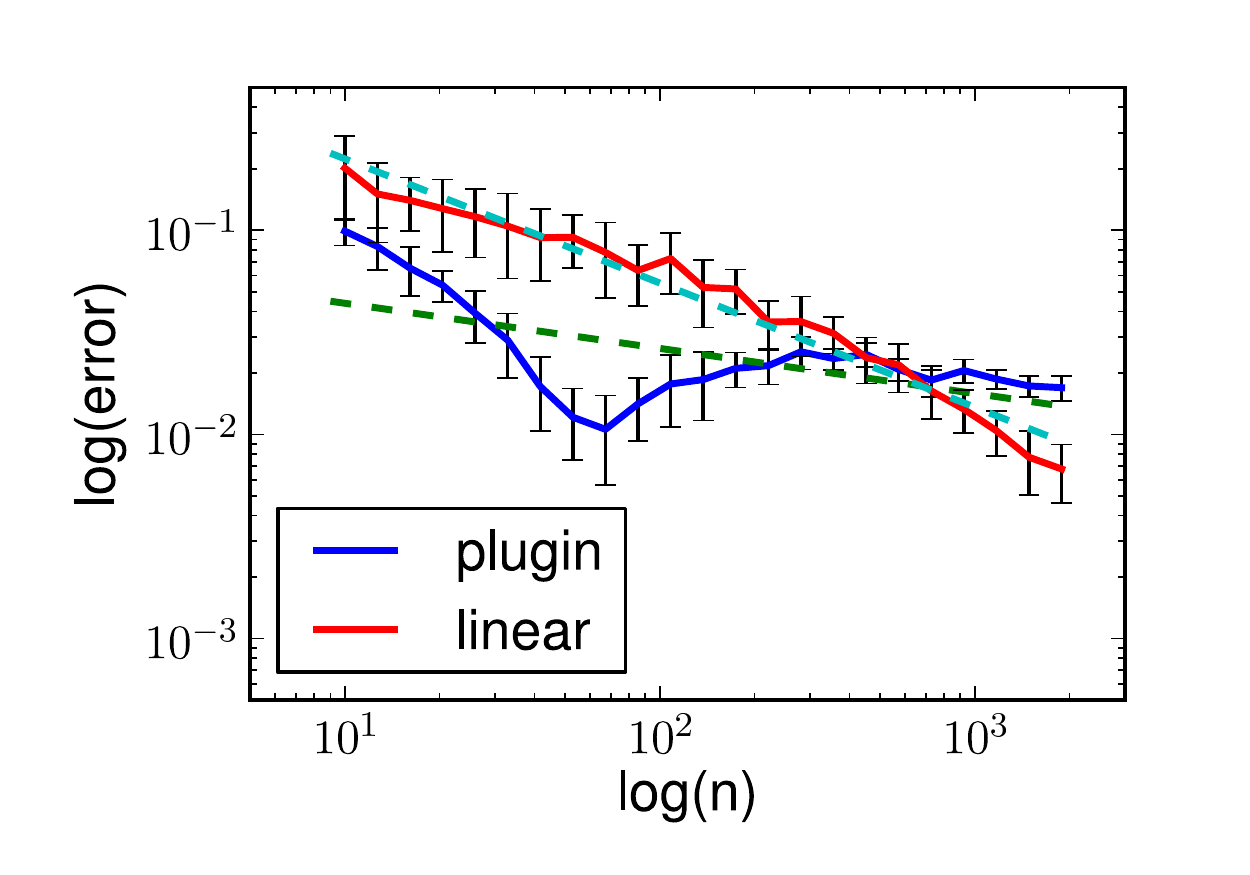}\\
\includegraphics[scale=0.3]{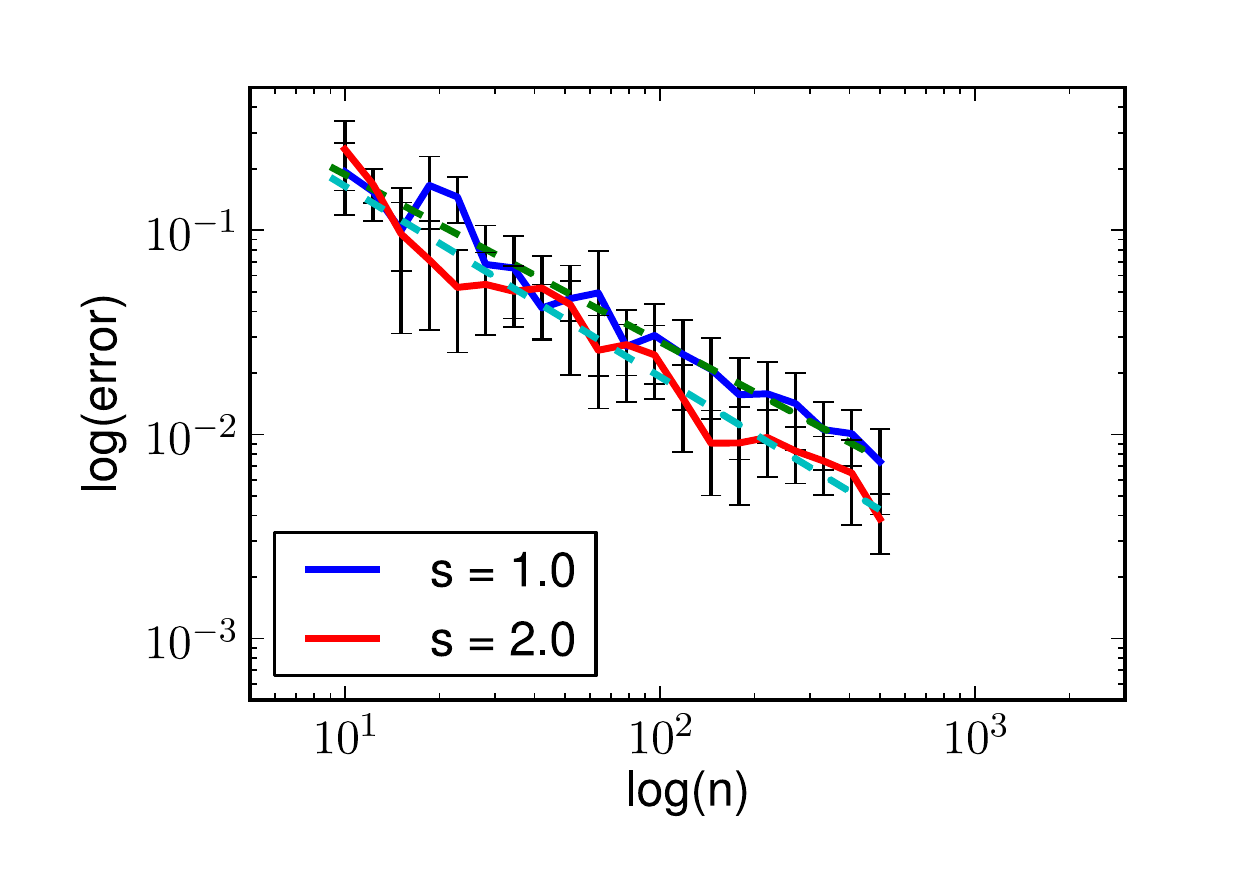} \hspace{-0.4cm}
\includegraphics[scale=0.3]{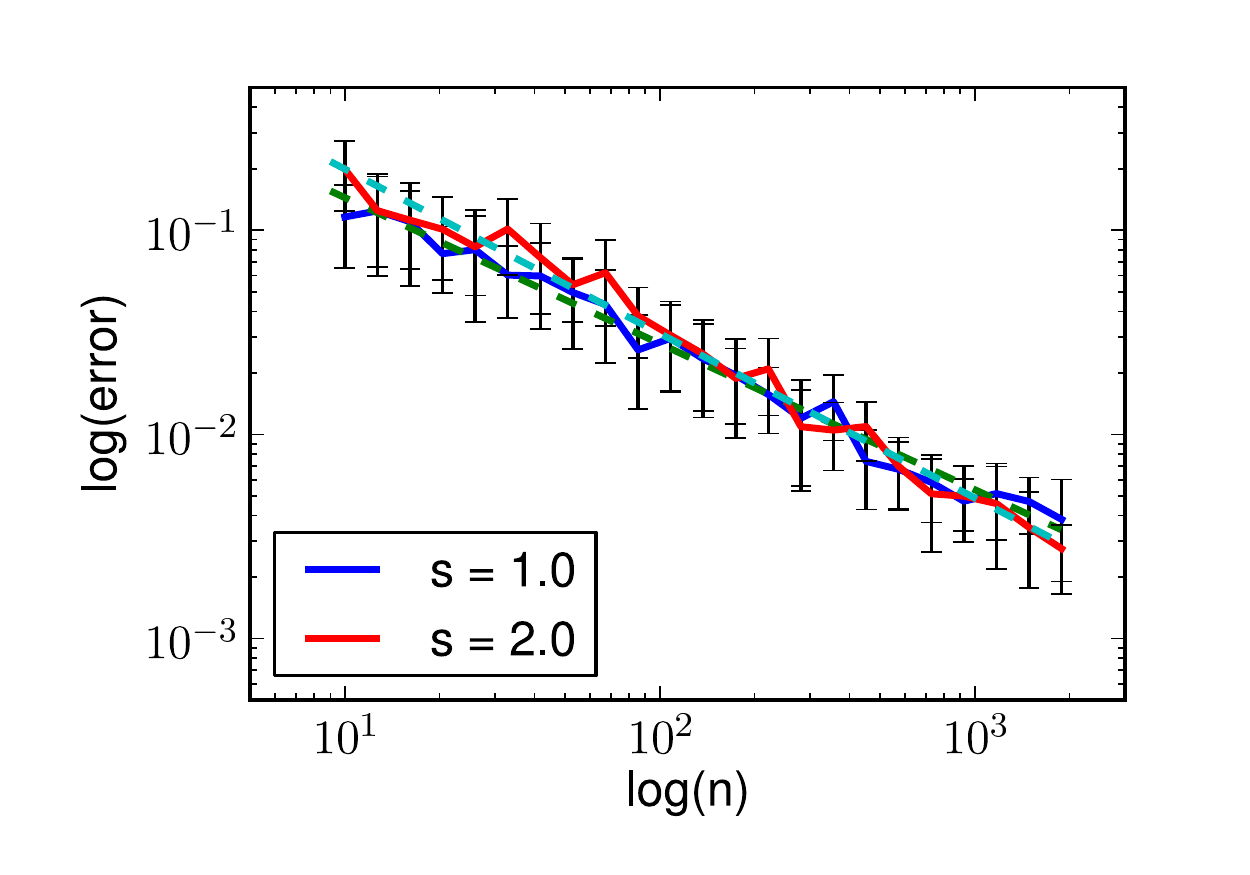} \hspace{-0.4cm}
\includegraphics[scale=0.3]{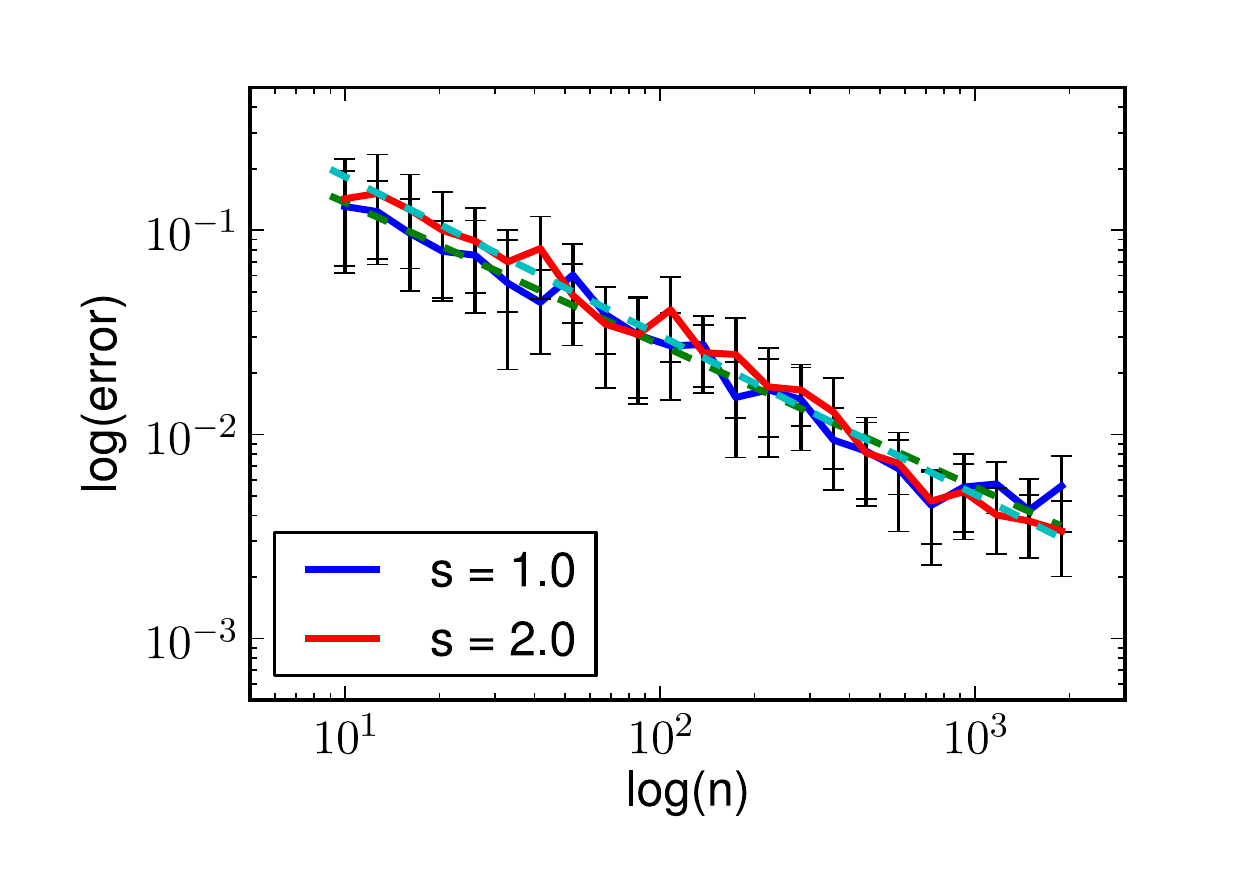} \hspace{-0.4cm}
\includegraphics[scale=0.3]{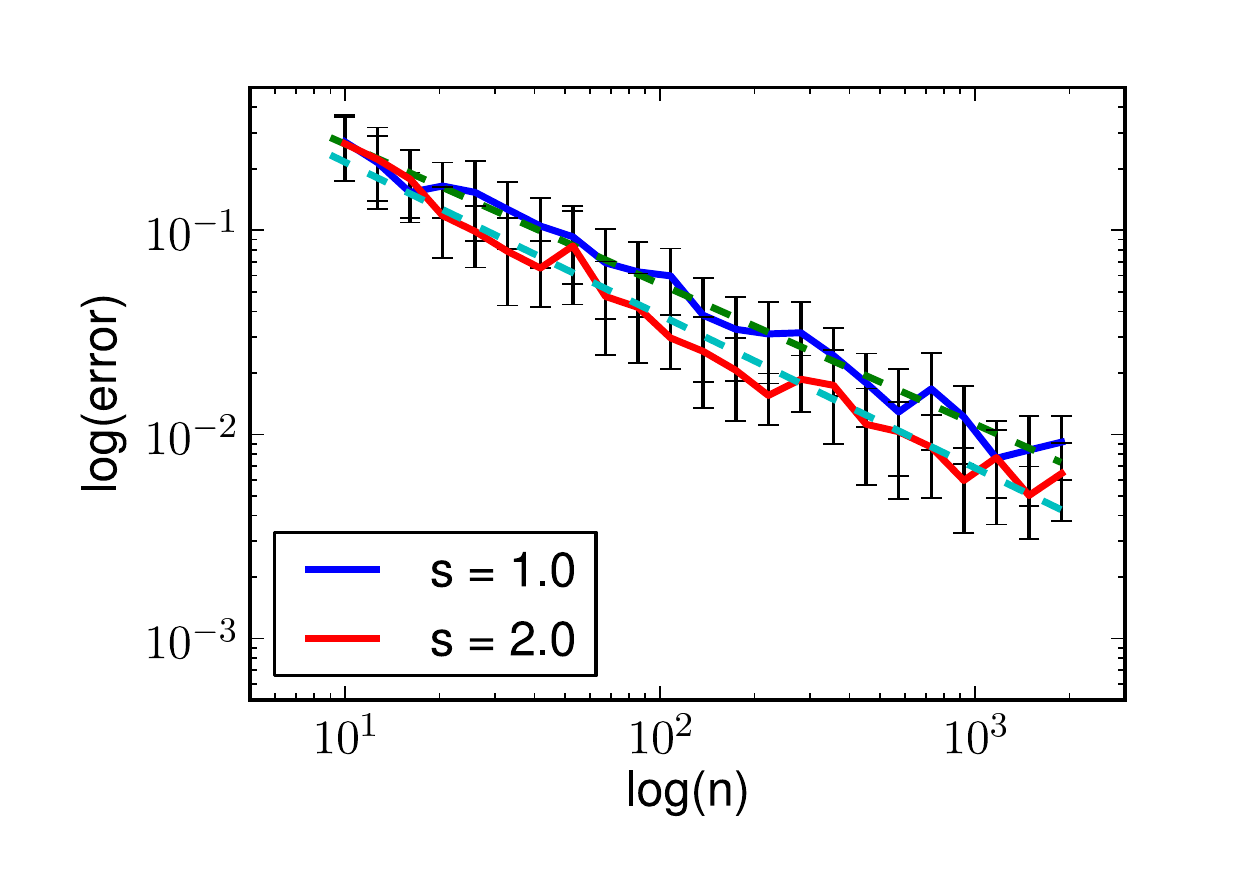}
\end{center}
\caption{Top row: Rates of convergence for $\Thatplugin, \Thatlin$ on a $\log$-$\log$ scale for: left: $d=1, s=1$, second from left: $d=1,s=2$, second from right: $d=2, s=2$, right: $d=2,s=4$.
Bottom Row: Left: Rate of convergence for $\Thatquad$ with $d=1, s=1.0, 2.0$. Middle two: Rates for linear estimator of $D_{0.5}(p,q), T_{0.5}(p,q)$ (respectively).
Right: Rate for $L_2^2$ estimator.
Dashed lines are fitted to the curves.
}
\label{fig:experiments}
\end{figure}
}{
\begin{figure*}[t]
\begin{center}
\hspace{-.6cm}
\includegraphics[scale=0.37]{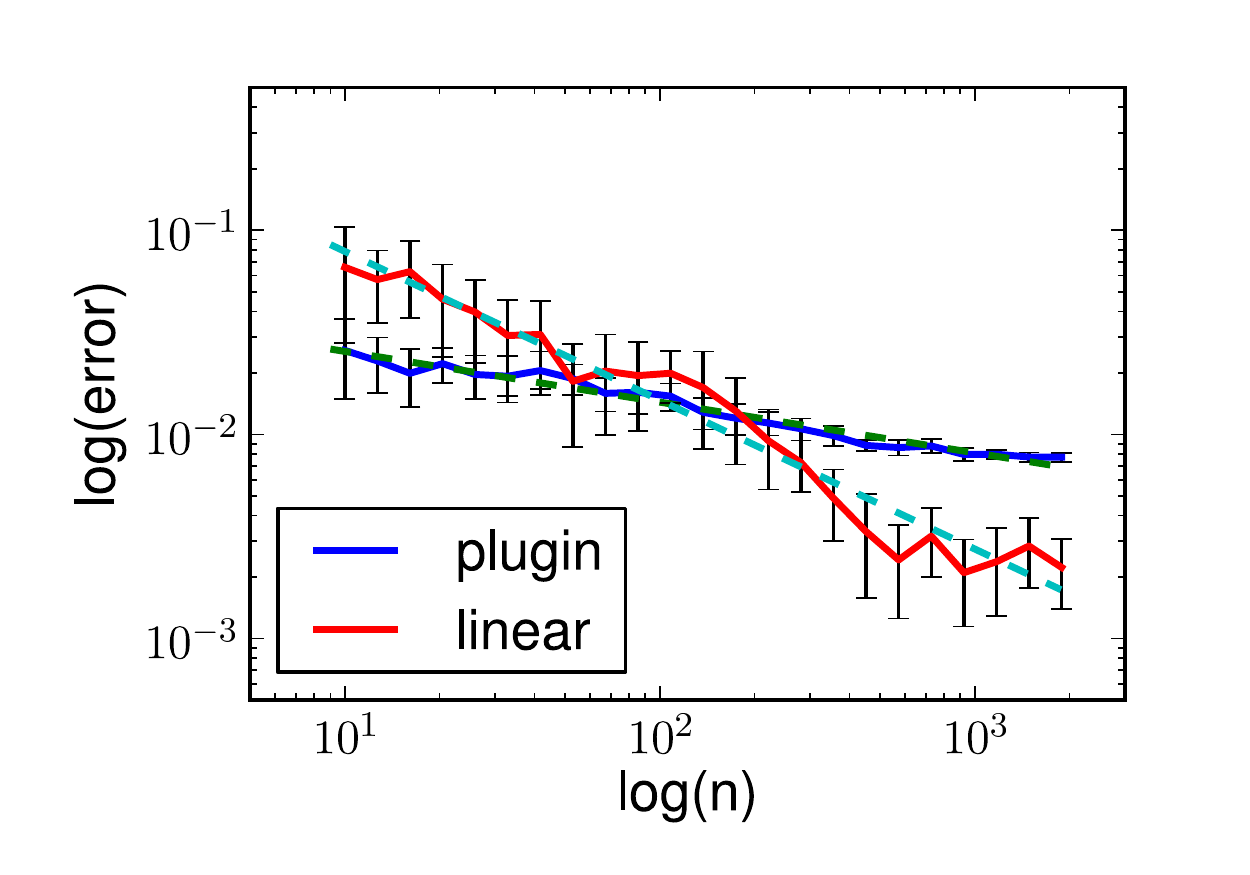} \hspace{-0.63cm}
\includegraphics[scale=0.37]{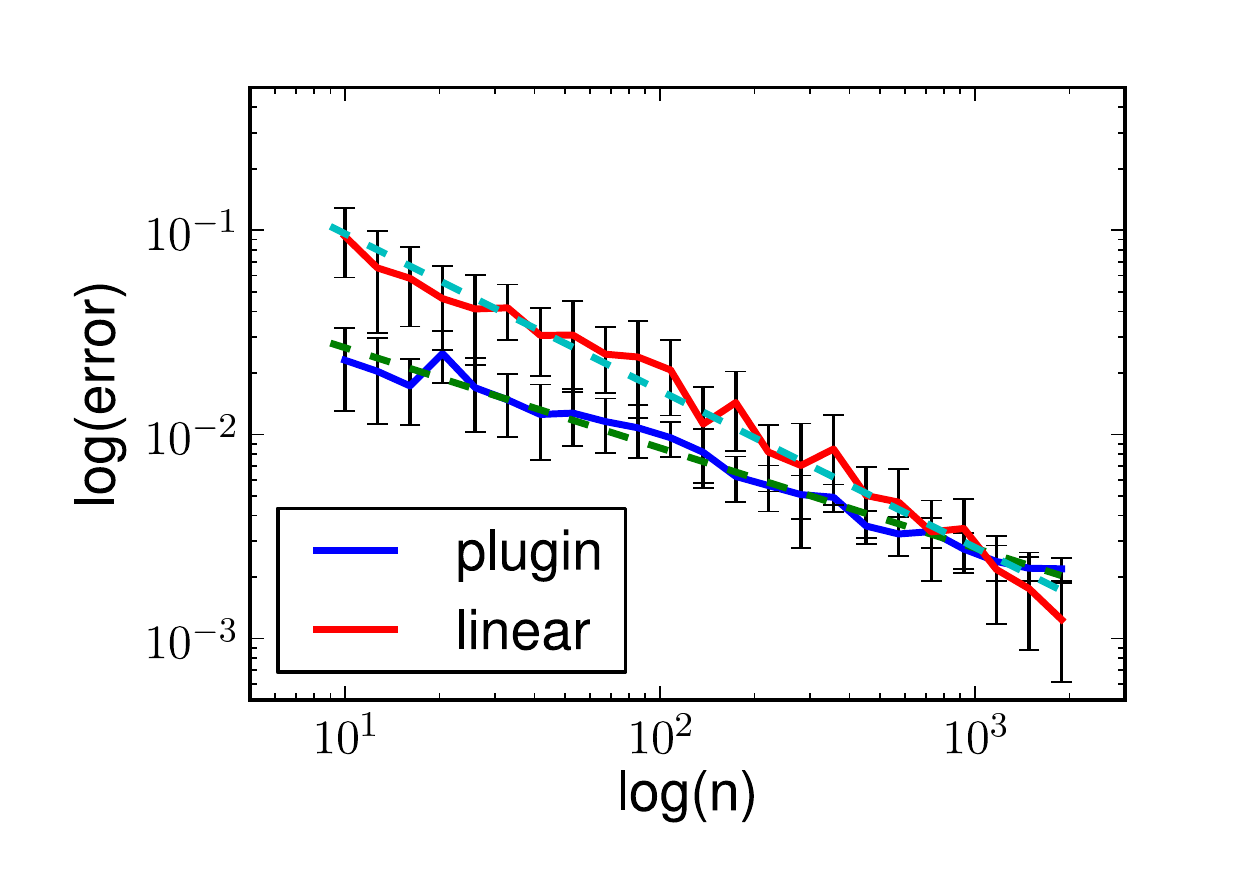} \hspace{-0.63cm}
\includegraphics[scale=0.37]{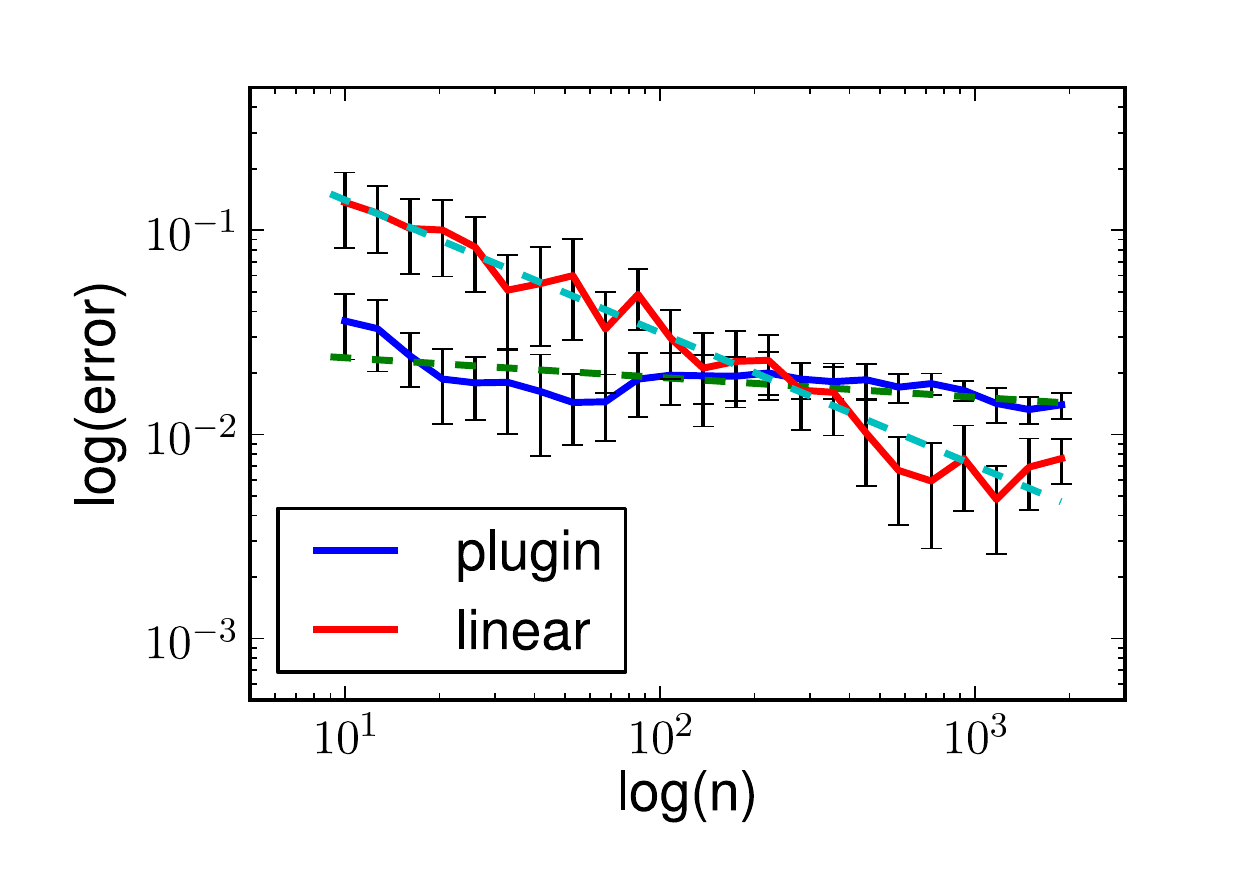} \hspace{-0.63cm}
\includegraphics[scale=0.37]{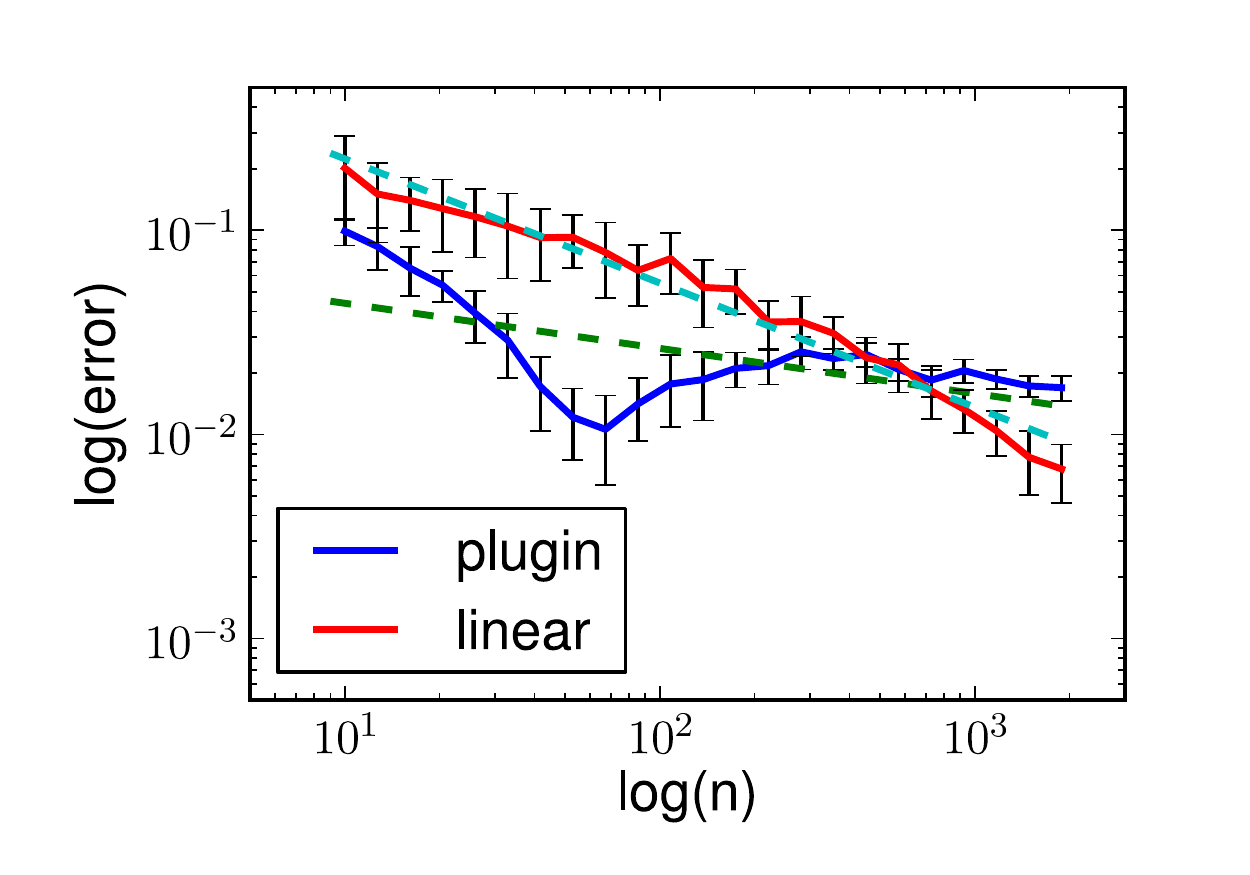}\\
\hspace{-0.6cm}
\includegraphics[scale=0.37]{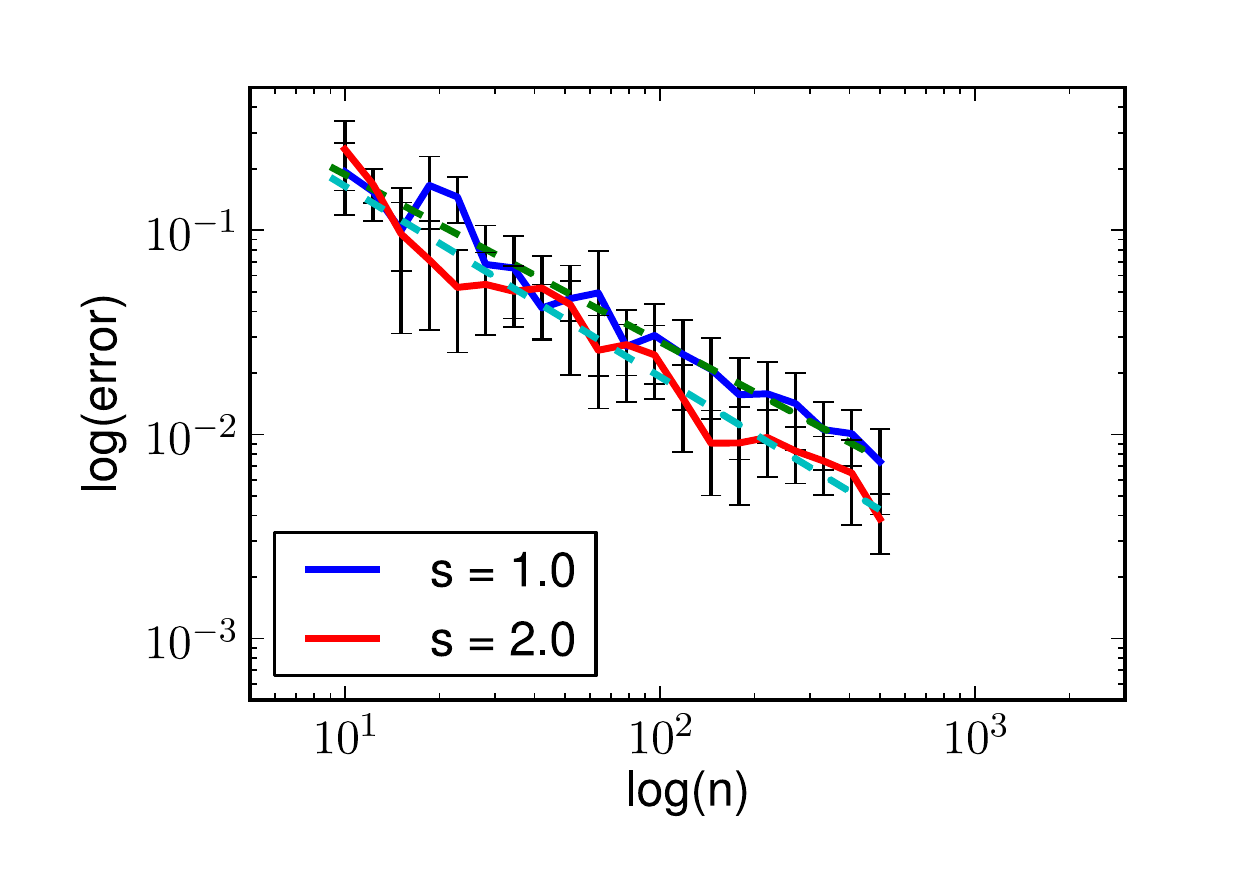} \hspace{-0.63cm}
\includegraphics[scale=0.37]{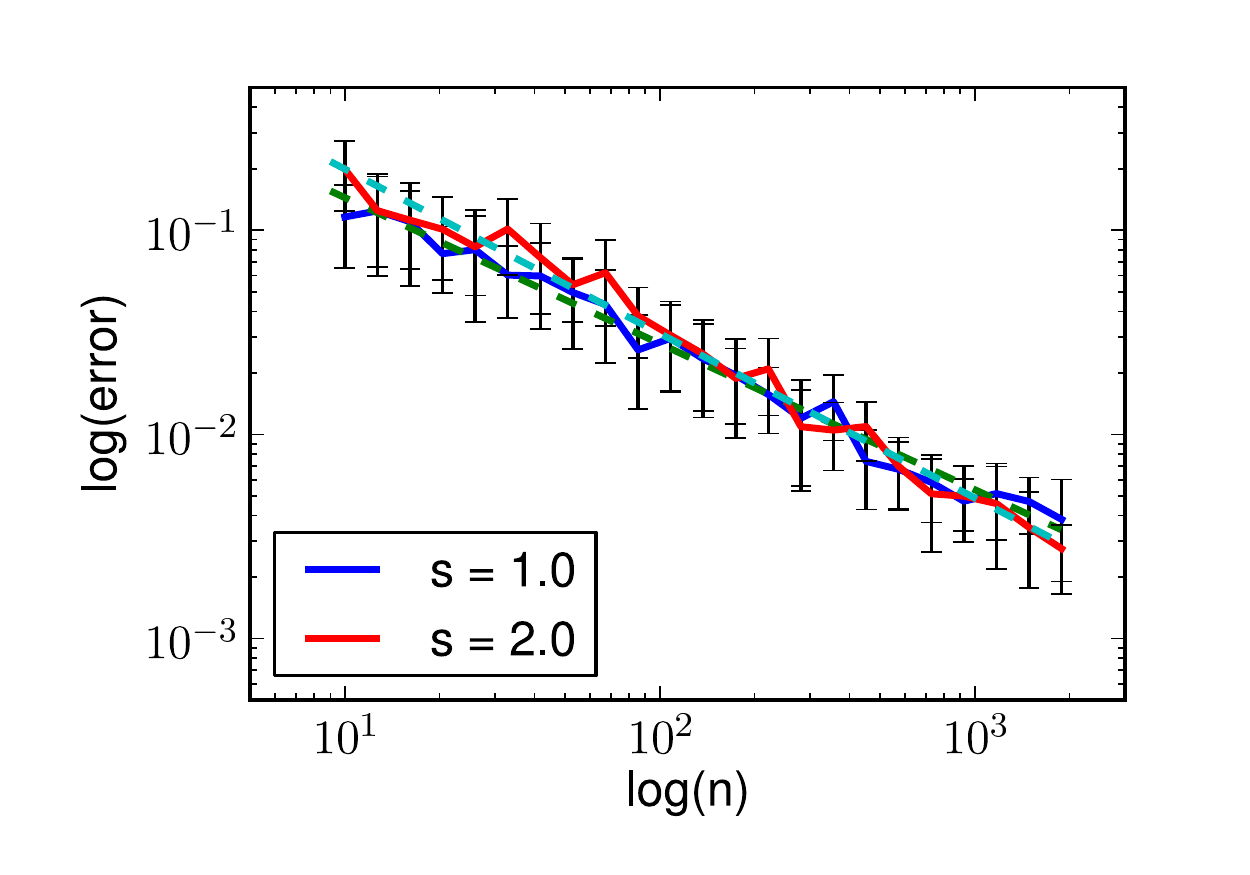} \hspace{-0.63cm}
\includegraphics[scale=0.37]{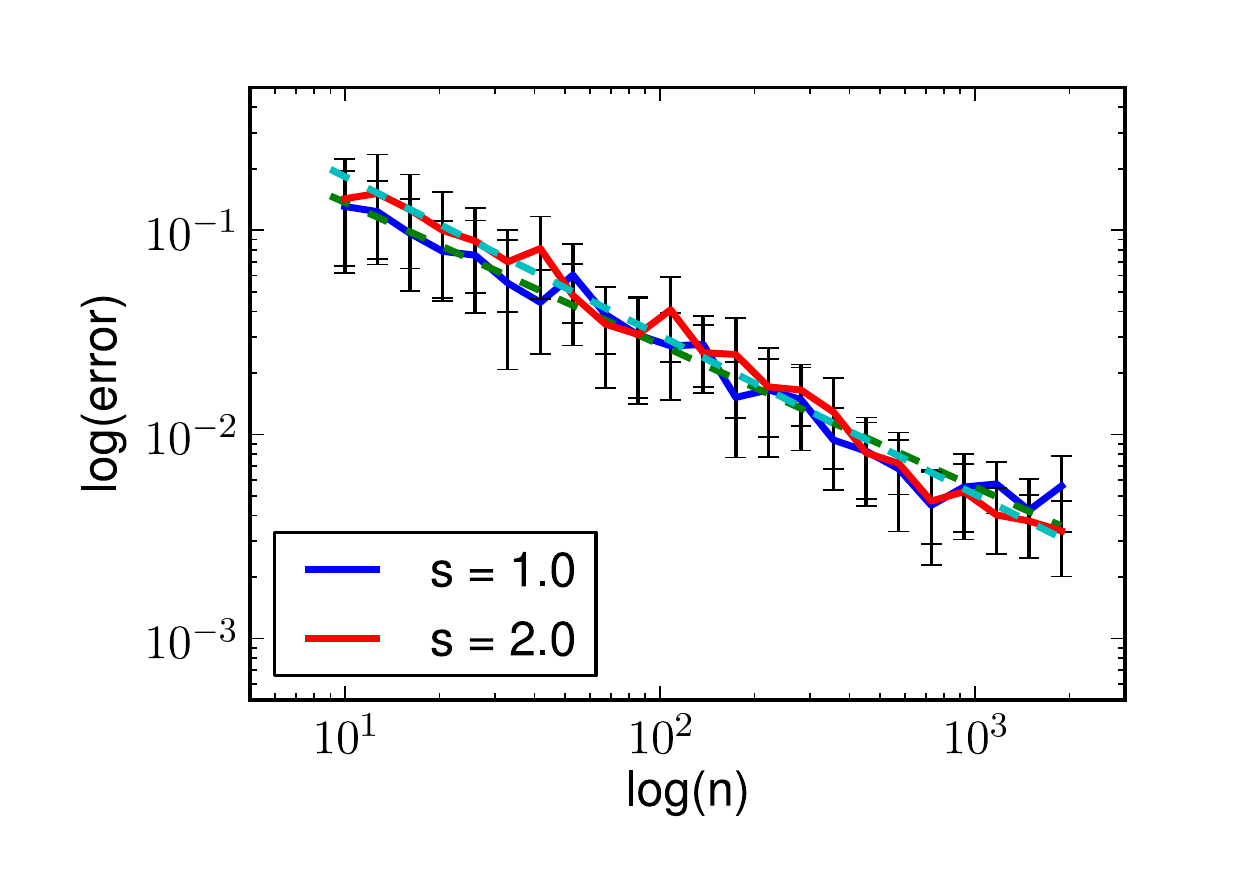} \hspace{-0.63cm}
\includegraphics[scale=0.37]{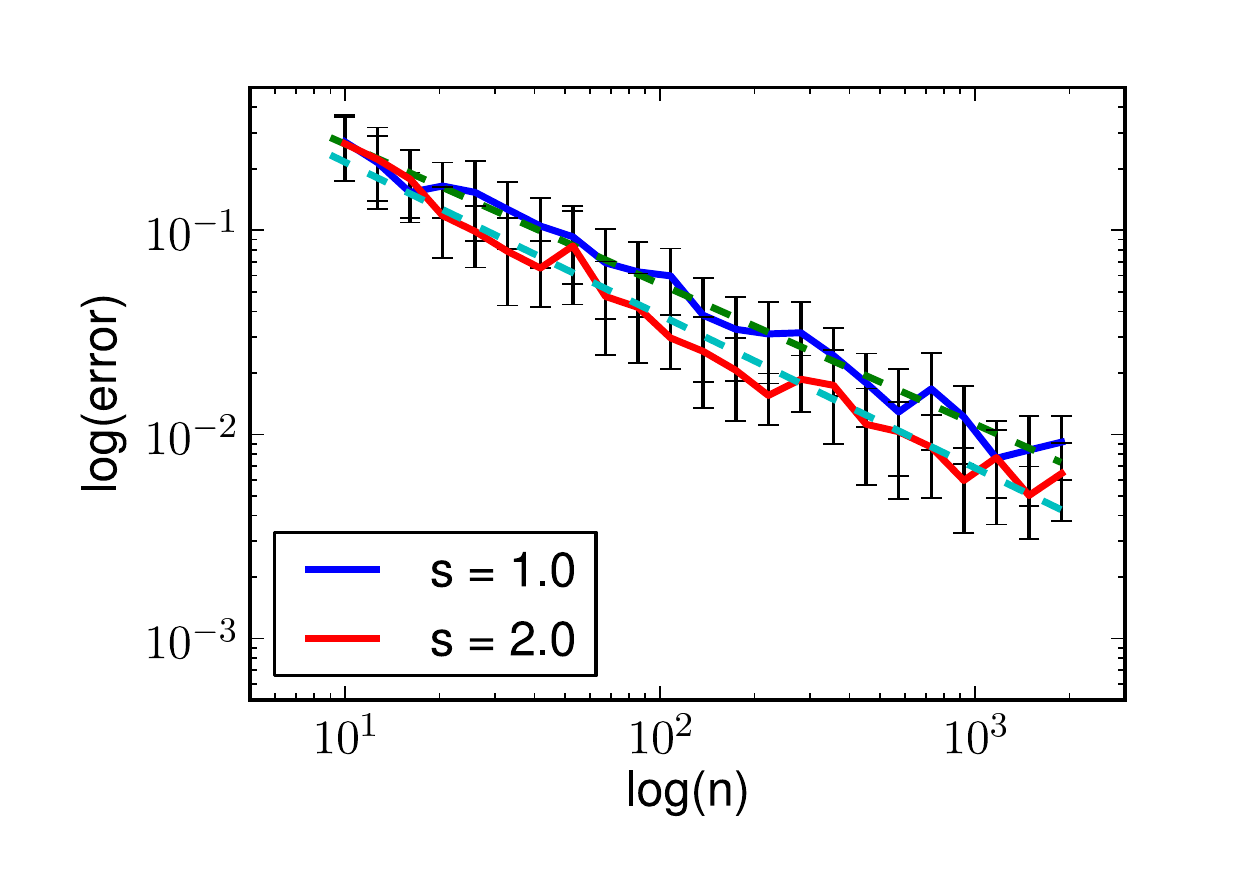}
\end{center}
\caption{Top row: Rates of convergence for $\Thatplugin, \Thatlin$ on a $\log$-$\log$ scale for: left: $d=1, s=1$, second from left: $d=1,s=2$, second from right: $d=2, s=2$, right: $d=2,s=4$.
Bottom Row: Left: Rate of convergence for $\Thatquad$ with $d=1, s=1.0, 2.0$. Middle two: Rates for linear estimator of $D_{0.5}(p,q), T_{0.5}(p,q)$ (respectively).
Right: Rate for $L_2^2$ estimator.
Dashed lines are fitted to the curves.
}
\label{fig:experiments}
\end{figure*}
}

We conducted some simulations to examine the empirical rates of convergence of our estimators.
We plotted the error as a function of the number of samples $n$ on a $\log$-$\log$ scale in Figure~\ref{fig:experiments} for each estimator and over a number of problem settings.
Since our theoretical results are asymptotic in nature, we are not concerned with some discrepancy between the empirical and theoretical rates.

In the top row of Figure~\ref{fig:experiments}, we plot the performance of $\Thatplugin$ and $\Thatlin$ across four different problem settings: $d=1,s=1$; $d=1, s=2$; $d=2, s=2$; and $d=2, s=4$. 
The lines fit to the plug-in estimator's error rate have slopes $-0.25, -0.5, -0.1, -0.2$ from left to right while the lines for the linear estimator have slopes $-0.7, -0.75, -0.65, -0.6$.
Qualitatively we see that the $\Thatlin$ is consistently better than $\Thatplugin$. 
We also see that increasing the smoothness $s$ appears to improve the rate of convergence of both estimators.

In the first plot on the bottom row, we record the error rate for $\Thatquad$ with $d=1$ and $s=1.0, 2.0$.
The fitted lines have slopes $-0.82, -0.93$ respectively, which demonstrate that $\Thatquad$ is indeed a better estimator than $\Thatlin$, at least statistically speaking.
Recall that we studied $\Thatquad$ primarily for its theoretical properties and to establish the critical smoothness index of $s > d/4$ for this problem.
Computing this estimator is quite demanding, so we did not evaluate it for larger sample size and in higher dimension.


Finally in the last three plots we show the rate of convergence for our divergence estimators, that is $\Thatlin$ plugged into the equations for $D_\alpha$ or $T_\alpha$ and the quadratic-based estimator for $L_2^2$.
Qualitatively, it is clear that the estimators converge fairly quickly and moreover we can verify that increasing the smoothness $s$ does have some effect on the rate of convergence.


%% file: discussion.tex
\section{Discussion}
\label{sec:discussion}
In this paper, we address the problem of divergence estimation with corrections of the plug-in estimator.
We prove that our estimators enjoy parametric rates of convergence as long as the densities are sufficiently smooth.
Moreover, through information theoretic techniques, we show that our best estimator $\Thatquad$ is nearly minimax optimal.

Several open questions remain. 
\begin{packed_enum}
\item Can we construct divergence estimators that are computationally \textrm{and} statistically efficient?
Recall that $\Thatquad$ involves numeric integration and is computationally impractical, yet $\Thatlin$, while statistically inferior, is surprisingly simple when applied to the divergences we consider.
At this point we advocate for the use of $\Thatlin$, in spite of its sub-optimality.
\item What other properties do these estimators enjoy?
Can we construct confidence intervals and statistical tests from them?
In particular, can we use our estimators to test for independence between two random variables?
\item Do our techniques yield estimators for other divergences, such as $f$-divergence and the Kullback-Leibler divergence?
\item Lastly, can one prove a lower bound for the case where $\alpha = \beta = 1$, i.e. the $L_2$ inner product?
\end{packed_enum}

We hope to address these questions in future work.

%% file: app_upper.tex
\section{The von Mises Expansion}
\label{app:von}
Before diving into the auxiliary results of Section~\ref{sec:proofs}, let us first derive some properties of the von Mises expansion. 
It is a simple calculation to verify that the Gateaux derivative is simply the functional derivative of $\phi$ in the event that $T(F) = \int \phi(f)$. 
\begin{lemma}
Let $T(F) = \int \phi(f)d\mu$ where $f = dF/d\mu$ is the Radon-Nikodym derivative, $\phi$ is differentiable and let $G$ be some other distribution with density $g = dG/d\mu$. 
Then:
\begin{align}
dT(G; F - G) =  \int \frac{\partial \phi(g(x))}{\partial g(x)} (f(x) - g(x))d\mu(x).
\end{align}
\label{lem:gateaux}
\end{lemma}
\begin{proof}
\begin{align*}
dT(G; F - G) &= \lim_{\tau \rightarrow 0} \frac{T(G + \tau(F-G)) - T(G)}{\tau} = \lim_{\tau \rightarrow 0} \int \frac{1}{\tau}\left[\phi(g(x) + \tau (f(x) - g(x))) - \phi(g(x))\right]d\mu(x)\\
& = \int \lim_{\tau \rightarrow 0} \frac{1}{\tau}\left[\phi(g(x) + \tau (f(x) - g(x))) - \phi(g(x))\right]d\mu(x)\\
& = \int \frac{\partial \phi(g(x))}{\partial g(x)} (f(x) - g(x))d\mu(x)
\end{align*}
\end{proof}

We now demonstrate that the remainder for the $t$th order von Mises expansion is $O(\|p-\phat\|_{t+1}^{t+1} + \|q-\qhat\|_{t+1}^{t+1})$ under the assumption that $p,\phat,q,\qhat$ are all bounded above and below. 
\begin{lemma}
Let $T(p,q) = \int p^\alpha q^\beta d\mu$ and uppose that $p, \phat, q, \qhat$ are all bounded from above and below. 
Then $R_t$, the remainder of the $t$th order von Mises expansion of $T(p,q)$ around $T(\phat,\qhat)$ satisfies:
\begin{align}
R_t = O\left(\|\phat - p\|_t^t + \|\qhat - q\|_t^t\right)
\end{align}
\label{lem:remainder}
\end{lemma}
\begin{proof}
The $t$th order term in the von Mises expansion is:
\begin{eqnarray*}
\frac{1}{t!} \sum_{a=0}^{t} {t \choose a} \int  \frac{\partial^{t} \phat^{\alpha}(x)\qhat^{\beta}}{\partial \phat(x)^{a} \partial \qhat(x)^{t-a}} (p(x)-\phat(x))^a(q(x)-\qhat(x))^{t-a}dx = \\
\frac{1}{t!} \sum_{a=0}^{t} {t \choose a} \int \prod_{i=0}^a (\alpha-i) \prod_{i=0}^{t-a} (\beta-i) \phat^{\alpha-a}(x)\qhat^{\beta-(t-a)}(x) (p(x)-\phat(x))^a(q(x)-\qhat(x))^{t-a}dx,
\end{eqnarray*}
where $\prod_{i=0}^0 a_i = 1$. 
If we are to take a $t-1$st order expansion, the remainder is of the same form as the $t$th term, except that the terms $\phat^{\alpha-a}(x), \qhat^{\beta-(t-a)}(x)$ are replaced by functions $\xi_1^{\alpha-a}(x), \xi_2^{\beta-(t-a)}(x)$ for some functions $\xi_1, \xi_2$ that are bounded between $p, \phat$ and $q,\qhat$ respectively.
In our setting, $p,q \in [\kappa_l, \kappa_u]$ and $\phat,\qhat \in [\kappa_l-\epsilon, \kappa_u+\epsilon]$ so $\xi_1, \xi_2$ are bounded functions.
With this bound, we can simplify the remainder term $R_{t-1}$ to:
\[
R_{t-1} \le C(\alpha,\beta, \kappa_l, \kappa_u, \epsilon, t) \frac{1}{t!}\sum_{a = 0}^t {t \choose a} \int |p(x) - \phat(x)|^a |q(x)- \qhat(x)|^{t-a} dx.
\]
Looking at the integral pointwise, either $|p(x)-\phat(x)| \le |q(x) - \qhat(x)|$ in which case the expression is upper bounded by $|q(x)- \qhat(x)|^t$ or the opposite is true in which case it is bounded by $|p(x)-\phat(x)|^t$.
Either way, we can upper bound the integral by the sum.
This gives:
\[
R_{t-1} \le C(\alpha,\beta, \kappa_l, \kappa_u, \epsilon, t) \frac{2^t}{t!} \left(\|p-\phat\|_t^t + \|q-\qhat\|_t^t\right).
\]
\end{proof}

In many cases, the constant can be worked out:
\begin{enumerate}
\item If $\alpha = \beta = 1$, then $R_1 = \alpha\beta$ while $R_{2}, \ldots, = 0$. 
\item If $\alpha,\beta > 0, \alpha+\beta=1$ as in the R\'{e}nyi Divergence, $R_2 = 1$ while $R_3 = \frac{5}{6}\kappa_{\epsilon}^{-2} \alpha \beta$ where $\kappa_{\epsilon} = \min\{\kappa_l-\epsilon, (\kappa_u+\epsilon)^{-1}\}$.
\end{enumerate}

The first order von Mises expansion is:
\begin{eqnarray*}
T(p,q) &=& T(\phat, \qhat) + \int \frac{\partial \phat^\alpha(x) \qhat^\beta(x)}{\partial \phat(x)}(p(x) - \phat(x)) + \int \frac{\partial \phat^\alpha(x) \qhat^\beta(x)}{\partial \qhat(x)}(q(x) - \qhat(x)) + O(\|p - \phat\|_2^2+ \|q - \qhat\|_2^2)\\
&=& T(\phat, \qhat) + \int \alpha \phat^{\alpha-1}(x) \qhat^\beta(x)(p(x) - \phat(x)) + \int \beta \phat^\alpha(x) \qhat^{\beta-1}(x)(q(x) - \qhat(x)) + O(\|p - \phat\|_2^2+ \|q - \qhat\|_2^2)\\
&=& (1-\alpha-\beta)T(\phat,\qhat) + \int \alpha \phat^{\alpha-1}(x) \qhat^\beta(x)p(x) + \int \beta \phat^\alpha(x) \qhat^{\beta-1}(x)q(x) + O(\|p - \phat\|_2^2+ \|q - \qhat\|_2^2)\\
&=& C_1T(\phat,\qhat) + \theta_{1,1}^p + \theta_{1,1}^q + R_2.
\end{eqnarray*}

The second order expansion is computed similarly.
The three second order terms are:
\begin{eqnarray*}
&&\frac{1}{2} \int \alpha(\alpha-1)\phat^{\alpha-2}(x)\qhat^\beta(x)(p(x) - \phat(x))^2\\
&&\int \alpha\beta \phat^{\alpha-1}(x)\qhat^{\beta-1}(x)(p(x) - \phat(x))(q(x) - \qhat(x))\\
&&\frac{1}{2}\int \beta(\beta-1)\phat^{\alpha}(x)\qhat^{\beta-1}(x)(q(x) - \qhat(x))^2.
\end{eqnarray*}
Adding these together along with the linear terms, expanding and regrouping terms we get:
\[
T_2(p,q) = C_2 T(\phat,\qhat) + \sum_{i=1,2 \atop f = p,q}\theta_{2,i}^f + \theta_{2,2}^{p,q} + R_3.
\]

\section{Full Specification of the Estimators}
\label{app:est}
Here we write out the complete expressions for the estimators $\Thatplugin, \Thatlin, \Thatquad$.
Recall that we have samples $X_1^n \sim p, Y_1^n \sim q$ and our goal is to estimate $T(p,q) = \int p^\alpha q^\beta$. 
Define:
\begin{align*}
\phat(x) = \frac{1}{n} \sum_{i=1}^n K_h(X_i - x) & \qquad \qhat(x) = \frac{1}{n} \sum_{j=1}^n K_h(Y_j - x)\\
\phat_{DS}(x) = \frac{2}{n} \sum_{i=1}^{n/2} K_h(X_i - x) & \qquad \qhat_{DS}(x) = \frac{2}{n} \sum_{j=1}^{n/2} K_h(Y_j - x),
\end{align*}
where $DS$ is used to denote that we are data splitting, and $K_h$ is a kernel with bandwidth $h$ meeting Assumption~\ref{ass:kernel}.
The estimator $\Thatplugin$ is formed by simply plugging in $\phat,\qhat$ into the function $T$. 
Formally:
\begin{align}
\Thatplugin = \int \phat^\alpha(x) \qhat^\beta(x) d\mu(x)
\end{align}
The estimator $\Thatlin$ is formed by a first order correction but we must used the data split KDEs to ensure independence between the multiple terms in the estimator.
\begin{align}
\Thatlin = (1 - \alpha - \beta) \int \phat_{DS}^\alpha(x) \qhat_{DS}^\beta(x) d\mu(x) + \frac{2}{n}\sum_{i=n/2+1}^{n} \alpha \phat_{DS}^{\alpha-1}(X_i)\qhat_{DS}^\beta(X_i) + \frac{2}{n}\sum_{j=n/2+1}^{n} \alpha \phat_{DS}^{\alpha}(Y_j)\qhat_{DS}^\beta(Y_j).
\end{align}

For the quadratic term we perform an additional correction:
\begin{align*}
\Thatquad &= (1 - 3\alpha/2 - 3\beta/2 + 1/2(\alpha+\beta)^2) \int \phat_{DS}^\alpha(x) \qhat_{DS}^\beta(x) d\mu(x) + \\
& + \frac{2}{n}\sum_{i=n/2+1}^n\alpha (2-\alpha-\beta)\phat_{DS}^{\alpha-1}(X_i)\qhat_{DS}^{\beta}(X_i) + \frac{2}{n}\sum_{j=n/2+1}^n\beta (2-\alpha-\beta)\phat_{DS}^{\alpha}(Y_j)\qhat_{DS}^{\beta-2}(Y_j)\\
& + \frac{4}{n(n/2-1)}\sum_{k \in M}\sum_{i_1 \ne i_2=n/2+1}^n \phi_k(X_{i_1})\phi_k(X_{i_2})\left[\frac{1}{2}\alpha(\alpha-1)\phat_{DS}^{\alpha-2}(X_{i_2})\qhat_{DS}^\beta(X_{i_2})\right]\\
& - \frac{2}{n(n/2-1)}\sum_{k,k' \in M}\sum_{i_1 \ne i_2=n/2+1}^n \phi_k(X_{i_1})\phi_{k'}(X_{i_2}) \left[\frac{1}{2}\alpha(\alpha-1) \int \phi_k(x)\phi_{k'}(x)\phat_{DS}^{\alpha-2}(x)\qhat_{DS}^\beta(x)d\mu(x)\right]\\
& + \frac{4}{n(n/2-1)}\sum_{k \in M}\sum_{j_1 \ne j_2=n/2+1}^n \phi_k(Y_{j_1})\phi_k(Y_{j_2})\left[\frac{1}{2}\beta(\beta-1)\phat_{DS}^{\alpha}(Y_{j_2})\qhat_{DS}^{\beta-2}(Y_{i_2})\right]\\
& - \frac{2}{n(n/2-1)}\sum_{k,k' \in M}\sum_{j_1 \ne j_2=n/2+1}^n \phi_k(Y_{j_1})\phi_{k'}(Y_{j_2}) \left[\frac{1}{2}\beta(\beta-1) \int \phi_k(y)\phi_{k'}(y)\phat_{DS}^{\alpha}(y)\qhat_{DS}^{\beta-2}(y)d\mu(y)\right]\\
& + \frac{2}{n}\sum_{j=n/2+1}^n \sum_{k \in M} \left(\frac{2}{n}\sum_{i=n/2+1}^n \phi_k(X_i)\right)\phi_k(Y_j) \left(\alpha \beta \phat_{DS}^{\alpha-1}(Y_j) \qhat_{DS}^{\beta-1}(Y_j)\right).
\end{align*}
Recall that $\{\phi_k\}_{k \in D}$ is an orthonormal basis for $L_2([0,1]^d)$, and $M$ is an appropriately chosen subset of $D$. 
The first line of the estimator is simply the plugin term, while the second lines makes up the two linear terms. 
The third through sixth lines are the two quadratic terms, one involving the data from $p$ and the other involving the data from $q$. 
Finally the last line is the bilinear term.

\section{Detailed Proofs of Upper Bound}
Let us now prove the the auxiliary results stated in Section~\ref{sec:proofs}

\subsection{Proof of Theorem~\ref{thm:kde_rate}}
The truncated kernel density estimator takes the following form: We select a parameter $\epsilon > 0$. 
If $\tilde{f}$ is the usual kernel density estimator for $f$, we set $\fhat(x)= \tilde{f}(x)$ if $\tilde{f}(x) \in [\kappa_l - \epsilon, \kappa_u+\epsilon]$ and otherwise we set $\fhat(x) = f_0(x)$ for some fixed function bounded between $\kappa_l, \kappa_u$. 

Recall Assumption~\ref{ass:kernel} ensures that the kernel $K:\RR^d \rightarrow \RR$ satisfies:
\begin{enumerate}
\item $\textrm{supp} K \in (-1,1)^d$
\item $\int K(x)dx = 1$ and $\int \prod_{i} x_i^{p_i} K(x)dx = 0$ for all tuples $p = (p_1, \ldots, p_d)$ with $\sum p_i \le \lfloor s \rfloor$. 
\end{enumerate}
Note that we can use the Legendre polynomials to construct kernels meeting these properties~\cite{tsybakov2009introduction}.

Let us first establish the rate of convergence of $\tilde{f}$ the regular kernel density estimator in $\ell_p^p$, which is $\tilde{f}(x) = \frac{1}{nh^d} \sum_{i=1}^nK\left(\frac{x - X_i}{h}\right)$. 
Denote by $\bar{f}(x) = \EE[\tilde{f}(x)] = \EE_{X \sim f} [\frac{1}{h^d}K\left(\frac{x-X}{h}\right)]$. 
Then:
\[
\EE[ \|\tilde{f} - f\|_p^p ] \le 2^p \left(\EE[ \|\tilde{f}- \bar{f}\|_p^p] + \|\bar{f} - f\|_p^p\right).
\]
To bound the first term, let us write $\eta_i(x) = \frac{1}{h^d} K\left(\frac{x - X_i}{h}\right) - \EE_{X \sim f}[\frac{1}{h^d}K\left(\frac{x-X}{h}\right)]$. 
Exchanging integrals, we can look at fixed $x$ and we have:
\begin{eqnarray}
\EE |\tilde{f}(x) - \bar{f}(x)|^p = \EE | \frac{1}{n} \sum_{i=1}^n \eta_i(x) |^p \le \left(\frac{1}{n^{2p}} \EE[(\sum_{i=1}^n \eta_i(x))^{2p}]\right)^{1/2}.
\label{eq:kde_variance_1}
\end{eqnarray}
If we expand the expectation and drop the terms that vanish we get all terms of the form:
\[
\sum_{i_1 \ne i_2 \ldots \ne i_t=1}^n\prod_{j=1}^t {p-\sum_{k=1}^{j-1}p_k \choose p_j} \eta_{i_j}(x)^{p_j} = \frac{n!}{(n-t)!}\prod_{j=1}^t {p-\sum_{k=1}^{j-1}p_k \choose p_j} \eta_{i_j}(x)^{p_j},
\]
where $1 \le t < p$, $\sum p_j = p$ and $p_j \ne 1 \forall j$.
That is, we pick a term in the polynomial with $t$ unique variables, then assign powers $p_j$ to each of the terms, then count the number of ways to assign those powers to those terms (which results in the binomial coefficients). 
Since $\EE[\eta_j(x)] = 0$, the terms where there is some $p_j = 1$ are all zero.

By linearity of expectation and independence, we therefore need to control $\EE[|\eta_i(x)|^q]$ for $2 \le q \le p$. 
Applying Jensen's inequality, we get:
\[
\EE[|\eta_i(x)|^q] \le 2^q\EE[|\frac{1}{h^d}K\left(\frac{X- x}{h}\right)|^q] \le 2^q \kappa_uh^{-(q-1)d} \int |K(u)^q| du,
\]
where the last expression comes from expanding the integral, performing a substitution and bounding $f(x) \le \kappa_u$. 
So we can bound by $C(q,\kappa_u, K) h^{-(q-1)d}$. 
Plugging this into the expression above, we get:
\[
\frac{n!}{(n-t)!} C(p_1^j, \kappa_u, K) h^{-pd+td} \le n^{t} C'(p_1^j, \kappa_u, K) h^{-pd+td} \le C'(p_1^j, \kappa_u, K) \frac{n^{p}}{h^{pd}}.
\]
The second inequality holds for $n$ sufficiently large.
The third inequality holds whenever $nh^{d} \ge 1$ which will be true for $n$ sufficiently large, given our setting of $h$. 
To summarize, all of the terms can be upper bounded by $c(n^{p}/h^{pd})$ and there are a constant-in-$p$ number of terms.
Plugging this into Equation~\ref{eq:kde_variance_1} we get
\begin{eqnarray}
\EE[\|\tilde{f} - \bar{f}\|_p^p] \le C (nh^d)^{-p/2}.
\label{eq:kde_variance_bound}
\end{eqnarray}
\begin{remark}
The constant here has exponential dependence on $p$ but we are only concerned with cases where $p$ is a small constant (at most $4$).
\end{remark}

As for the bias (note that $x,u,t$ are all $d$-dimensional vectors here):
\[
|\bar{f}(x) - f(x)| = \int \frac{1}{h^d} K\left(\frac{x - t}{h}\right) f(t)dt - f(x) = \int (f(x-uh) + f(x))K(u)du.
\]
Let us define $m = \lfloor s \rfloor$. 
Taking the $(m-1)$st order von Mises expansion of $f(x+uh)$ about $f(x)$ we get terms of the form:
\[
\sum_{r_1,\ldots, r_d | \sum r_i \le m-1} \frac{1}{|r|!} D^rf(x) h^{|r|} \int \prod_{i}u_i^{r_i} K(u) du
\]
which are all zero by our assumption on $K$.
The remainder term, gives us:
\[
\sum_{r_1, \ldots, r_d | \sum r_i = m} \frac{h^m}{m!} \int \xi(r, x,uh) \prod_i u_i^{r_i}K(u) du \le \sum_{r_1, \ldots, r_d | \sum r_i = m} \frac{Lh^s}{m!} \int \|u\|^{s - m} \prod_i u_i^{r_i} K(u) du,
\]
which we will denote $C(m, K, d) L h^s$.
Here the function $\xi$ is between $D^rf(x)$ and $D^rf(x-uh)$ and to reach the last expression, we use the fact that $|D^rf(x) - D^rf(x-uh)| \le L\|uh\|^{s-r}$, i.e. the H\"{o}lderian assumption on $f$. 
In applying the H\"{o}lderian assumption, there is another term of the form $D^rf(x)\int \prod_i u_i^{r_i}K(u)du$ which is zero by the assumption on $K$. 
Equipped with this bound, we can bound the bias:
\begin{eqnarray}
\|\bar{f} - f\|_p^p \le C(m,K,d)L^p h^{ps}.
\end{eqnarray}
In trading off the bias and the variance, we set $h \asymp n^{\frac{-1}{2s+d}}$ and see that the rate of convergence is $\EE[\|\tilde{f} - f\|_p^p = O(n^{\frac{-ps}{2s+d}})$.

To prove Theorem~\ref{thm:kde_rate}, we just have to show that truncation does not significantly affect the rate. 
Fix $\epsilon > 0$ and define $S_\epsilon = \{x: \kappa_l - \epsilon \le \tilde{f}(x) \le \kappa_u+\epsilon\}$. 
We have:
\begin{eqnarray*}
\EE[ \|\fhat - f\|_p^p ] &=& \EE \left[ \int_{S_\epsilon} |\tilde{f}(x) - f(x)|^p dx + \int_{S_\epsilon^C} |f_0(x) - f(x)| dx\right]\\
& \le & \EE\left[ \|\tilde{f} - f\|_p^p\right] + \|f_0 - f\|_{\infty}^p \EE\left[ \int \mathbf{1}[x \notin S_\epsilon] dx \right]\\
& = & \EE\left[ \|\tilde{f} - f\|_p^p\right] + \|f_0 - f\|_{\infty}^p \int \PP_{X_1^n}[x \notin S_\epsilon] dx,
\end{eqnarray*}
so we must control the probability that $x \notin S_\epsilon$.
This can be done via Bernstein's inequality. 
First observe that the bias $|\bar{f} - f| \rightarrow 0$ with our choice of $h$ so that for sufficiently large $n$, $\sup_x \bar{f}(x) - f(x) \le \epsilon/2$. 
Once this happens, it is clear that $x \notin S_\epsilon$ implies that $\tilde{f}(x) - \bar{f}(x) \ge \epsilon/2$. 
Therefore:
\begin{eqnarray*}
\PP[x \notin S_\epsilon] \le \PP[|\tilde{f}(x) - \bar{f}(x)| \ge \epsilon/2] = \PP[|\frac{1}{n}\sum_i \eta_i(x)| \ge \epsilon/2] \le 2 \exp\left(\frac{-nh^d\epsilon^2/4}{\kappa_u\|K\|_2^2 + \frac{1}{3}\|K\|_{\infty}\epsilon}\right).
\end{eqnarray*}

This last inequality is an application of Bernstein's inequality noting that $|\eta_i(x)| \le \frac{2}{h^d} \|K\|_{\infty}$ and $\Var(\eta_i(x)) \le h^{-d} \kappa_u \|K\|_{2}^2$ since:
\[
\Var(\eta_i(x)) \le \EE_{X_i \sim f} \left[\frac{1}{h^{2d}} K^2(\frac{X_i - x}{h})\right] = \frac{1}{h^d}\int K^2(u)f(x+hu)du \le h^{-d} \kappa_u \|K\|_2^2.
\]
Using our definition $h \asymp n^{\frac{-1}{2s+d}}$ and using the fact that $\epsilon$ is some constant $\PP[x \notin S_\epsilon] \le 2 \exp(-C n^{\frac{2s}{2s+d}})$.
Plugging this bound in above, we have:
\begin{eqnarray*}
\EE\left[\|\fhat - f\|_p^p\right] \le \EE\left[\|\tilde{f} - f\|_p^p\right] +  2\|f_0 - f\|_{\infty}^p \exp\left( - C n^{\frac{2s}{2s+d}} \right) \textrm{vol}([0,1]^d)
 = O(n^{\frac{-ps}{2s+d}}),
\end{eqnarray*}
since the second term goes to zero exponentially quickly in $n$. 
This proves the theorem.

\subsection{Convergence Rate for Estimating Linear Functionals}
It is trivial to derive the convergence rate for estimating linear functionals:
\begin{eqnarray*}
\EE[(\hat{\theta} - \theta)^2] = \frac{1}{n} (\EE[\psi^2(X)]- \EE[\psi(X)]^2) \le 2\|\psi\|_{\infty}^2/n,
\end{eqnarray*}
And by Jensen's inequality, we have $\EE[|\hat{\theta} - \theta|] \le \sqrt{\EE[(\hat{\theta} - \theta)^2]}$, so the rate of convergence is $\sqrt{2} \|\psi\|_{\infty}/\sqrt{n}$. 

\subsection{Proof of Theorem~\ref{thm:quadratic_rate}}
For the quadratic terms, we use a result of Laurent~\yrcite{laurent1996efficient}:
\begin{theorem}[\cite{laurent1996efficient}]
Let $X_1^n$ be i.i.d random variables with common density $f$ that belongs to some Hilbert Space $L^2(d\mu)$. 
Let $\{\phi_i\}_{i \in D}$ be an orthonormal basis of $L^2(d\mu)$.
Assume that $f$ is uniformly bounded and belongs to the ellipsoid $\Ecal = \{\sum_{i \in D} a_i\phi_i : \sum_{i \in D} |a_i^2/c_i^2| \le 1\}$.
Let $\psi$ be bounded function and define $\theta = \int \psi(x) f(x) \mu(dx)$ and $\hat{\theta}$ as in Equation~\ref{eq:quadratic_estimator} where the set $M = M_n \subset D$ has size $m$.
Then whenever $n \ge n_0$ (some absolute constant), we have:
\begin{eqnarray}
\EE[(\hat{\theta} - \theta)^2] = \textrm{Bias}^2(\hat{\theta}) + \Var(\hat{\theta}) \le \|\psi\|_{\infty}^2 \sup_{i \notin M_n} |c_i|^4 + 72 \|\psi\|_{\infty}^2\|f\|_{\infty}^2 \left(\frac{2}{n} + \frac{m}{n^2}\right).
\end{eqnarray}
\end{theorem}

For the bi-linear term $\theta_{2,2}^{p,q}$ we have the following theorem:
\begin{theorem}
Let $X_1^n$ be i.i.d random variables with common density $f$ and $Y_1^n$ be i.i.d. with common density $g$.
Let $f,g$ belong to some Hilbert space $L^2(d\mu)$ and let $\{\phi_i\}_{i \in D}$ be an orthonormal basis for $L^2(d\mu)$.
Assume that $f,g$ are uniformly bounded and both belong to the ellipsoid $\Ecal = \{\sum_{i \in D} a_i\phi_i : \sum_{i \in D} |a_i^2/c_i^2| \le 1\}$.
Let $\theta = \int \psi(x) f(x)g(x) \mu(dx)$ and $\hat{\theta}$ be defined by Equation~\ref{eq:both_estimator} where the set $M = M_n \subset D$ has size $m$.
Then whenever $n \ge n_0$ (some absolute constant), we have:
\begin{eqnarray}
\EE[(\hat{\theta} - \theta)^2] = \textrm{Bias}^2(\hat{\theta}) + \Var(\hat{\theta}) \le \|\psi\|_{\infty}^2 \sup_{i \notin M_n} |c_i|^4 + \|\psi\|_{\infty}^2\|f\|_{\infty}\|g\|_{\infty} \left(\frac{2}{n} + \frac{m+1}{n^2}\right).
\end{eqnarray}
\label{thm:bilinear_sub}
\end{theorem}

\begin{proof}
The bias is:
\begin{eqnarray*}
\EE[\hat{\theta}] - \theta = \int \sum_{i \in M} \alpha_i \phi_i(x) \psi(x) g(x) dx  - \int \psi(x) f(x) g(x) = \int \psi(x) \left(\Pcal_M f(x) - f(x)\right)g(x) dx,
\end{eqnarray*}
where $\alpha_i = \int \phi_i(x) f(x)$ and $\Pcal_M f$ is the projection of $f$ onto the subspace defined by $M$. 
Define $\beta_i = \int \phi_i(x) g(x)$.
If $f,g$ live in the ellipsoid $\Ecal = \{ \sum a_i \phi_i | \sum |a_i|^2/|c_i|^2 \le L\}$ then:
\begin{eqnarray*}
\textrm{Bias}^2(\hat{\theta}) = \left( \sum_{i \notin M} \alpha_i \int \psi(x) g(x) \phi_i(x) dx\right)^2 \le \|\psi\|_{\infty}^2 \left( \sum_{i \notin M} \alpha_i \beta_i \right)^2.
\end{eqnarray*}
The term inside the parenthesis can be bounded as:
\[
\sum_{i \notin M} \alpha_i \beta_i \le \frac{1}{2} \sup_{i \notin M} |c_i|^2 \sum_{i \notin M} \frac{|\alpha_i|^2 + |\beta_i|^2}{|c_i|^2} \le L \sup_{i \notin M} |c_i|^2,
\]
so the bias is $\textrm{Bias}^2(\hat{\theta}) \le \|\psi\|_{\infty}^2 L^2 \sup_{i \notin M} |c_i|^4$.

As for the variance, let us define $Q(x)$ to be the $m$-dimensional vector of functions $\phi_i(x) - \alpha_i$ and $R(x)$ to be the $m$-dimensional vector of functions $\phi_i(x)\psi(x) - \int \psi \phi_i g$.
Further define $A, B$ to be the $m$-dimensional vectors with $i$th components $\alpha_i = \int \phi_i f$ and $\beta_i = \int \psi \phi_i g$ respectively. 
Then our estimator can alternatively be written as:
\begin{eqnarray*}
\hat{\theta} = \underbrace{\frac{1}{n^2} \sum_{j,k} Q(X_j)^T R(Y_k)}_{T_1} + \underbrace{\frac{1}{n}\sum_{j} Q(X_j)^TB}_{T_2} + \underbrace{\frac{1}{n}\sum_{k} A^T R(Y_k)}_{T_3} - A^TB.
\end{eqnarray*}
Notice that $Q, R$ are centered functions.
Since $X$s are independent of the $Y$s, $\textrm{Cov}(T_2, T_3) = 0$.
Since $T_2$ is independent of $Y$ and $\EE[R(Y_k)] = 0$, we see that $\textrm{Cov}(T_1, T_2) = 0$.
Similarly, $\textrm{Cov}(T_1, T_3) = 0$. 

Therefore,
\[
\Var(\hat{\theta}) = \Var(T_1) + \Var(T_2) + \Var(T_3).
\]
Let us analyze $T_1$. 
By independence,
\begin{eqnarray*}
\Var(T_1) &=& \frac{1}{n^2} \Var(Q(X_1)^TR(Y_1))
= \frac{1}{n^2} \sum_{i,i' \in M} \int \phi_i(x) \phi_{i'}(x) \phi_i(y) \phi_{i'}(y) \psi(y)^2 f(x) g(y) dx dy\\
 &-& \int \alpha_i\alpha_{i'} \phi_{i}(y)\phi_{i'}(y)\psi(y)^2g(y)dy - \int \beta_{i}\beta_{i'} \phi_i(x) \phi_{i'}(x)f(x) + \alpha_i\alpha_{i'}\beta_i\beta_{i'}\\
& \le & \frac{1}{n^2} \sum_{i,i' \in M} \int \phi_i(x) \phi_{i'}(x) \phi_i(y) \phi_{i'}(y) \psi(y)^2 f(x) g(y) dx dy + \frac{1}{n^2} (\sum_i \alpha_i \beta_i)^2\\
& = & \frac{1}{n^2} \int \left(\sum_{i \in M} \phi_i(x) \phi_i(y)\right)^2 \psi(y)^2 f(x) g(y) dx dy + \frac{1}{n^2} (\sum_i \alpha_i \beta_i)^2\\
& \le & \frac{\|\psi\|_{\infty}^2\|f\|_{\infty} \|g\|_{\infty}}{n^2}  \int \left(\sum_{i \in M} \phi_{i}(x) \phi_i(y)\right)^2 dx dy + \frac{1}{n^2}\left(\sum_{i} \alpha_i^2\right)\left(\sum_{i} \beta_i^2\right)\\
& \le & \frac{\|\psi\|_{\infty}^2 \|f\|_{\infty} \|g\|_{\infty} m}{n^2} + \frac{1}{n^2}(\int f^2)(\int g^2 \psi^2)
\le \frac{\|\psi\|_{\infty}^2 \|f\|_{\infty} \|g\|_{\infty} (m+1)}{n^2} .
\end{eqnarray*}
To arrive at the third line, notice that the cross terms are non-negative, since $\sum_{i, i'} \alpha_i \alpha_{i'} \phi_i(y) \phi_{i'}(y) = \left(\sum_i \alpha_i \phi_i(y)\right)^2$ (and analogously for the other cross term).
Therefore we can simply omit them and provide an upper bound. 
To go from the fourth to fifth lines, we use H\"{o}lder's inequality on the first term and Cauchy-Schwarz on the second term. 
Notice that the expression involving $\phi_i(x) \phi_i(y)$ is positive, so we can drop the absolute values in the $\ell_1$ norm term of H\"{o}lder's inequality.
To arrive at the fifth line, we expand out the square and use the fact that $\phi_i$s are orthornormal.

For $T_2$ again by independence we have:
\begin{eqnarray*}
\Var(T_2) &=& \frac{1}{n} \Var(Q(X_1)^TB) = \EE[ \left(\sum_{i \in M} (\phi_i(X_1) - \alpha_i) \int \psi \phi_i g \right)^2]\\
& = & \sum_{i,i'\in M} \int \phi_i(x)\phi_{i'}(x) f(x) \int \psi \phi_i g \int \psi \phi_{i'} g - \int \alpha_i \psi \phi_i g \int \alpha_{i'} \psi \phi_{i'} g\\
& = & \int (\sum_{i \in M} \beta_i \phi_i(x))^2 f(x) - \left(\int (\Pcal_M ) \psi g\right)^2 \le \int (\Pcal_M(\psi g))^2 f.
\end{eqnarray*}
Here the last inequality follows from the fact that $\beta_i = \int \psi \phi_i g$ is the $i$th fourier coefficient of $\psi g$ so $\sum_i \beta_i \phi_i$ is the projection onto $M$. 
Of course this quantity is bounded by:
\[
\Var(T_2) \le \frac{1}{n} \|f\|_{\infty} \int \psi^2(x)g^2(x) dx \le \frac{\|\psi\|_{\infty}^2 \|f\|_{\infty} \|g\|_{\infty}}{n}.
\]
Essentially the same argument reveals that $T_3$ is bounded in the same way. 
\begin{eqnarray*}
\Var(T_3) & = & \frac{1}{n} \Var(A^T R(Y_1)) \le \frac{\|\psi\|_{\infty}^2}{n} \sum_{i,i'}\alpha_i \alpha_{i'} \left[\int \phi_i(y) \phi_{i'}(y)g(y)dy - \int \phi_i g \int \phi_{i'}g\right]\\
& = & \frac{1}{n} \|\psi\|_{\infty}^2  \left[\int (\Pcal_M f)^2 g - (\int (\Pcal_M f) g)^2\right] \le \frac{\|\psi\|_{\infty}^2\|f\|_{\infty} \|g\|_{\infty}}{n},
\end{eqnarray*}
so the variance of the estimator is:
\[
\Var(\hat{\theta}) \le \|\psi\|_{\infty}^2 \|f\|_{\infty} \|g\|_{\infty} \left( \frac{m+1}{n^2} + \frac{2}{n}\right).
\]
\end{proof}

Both the quadratic and bilinear terms exhibit the same dependence on $\sup_{i \notin M_n} |c_i|, m, n$ so choosing $M_n$ appropriately will give the rate of convergence for both terms. 
To establish Theorem~\ref{thm:quadratic_rate} we work with the fourier basis $\{\phi_k\}_{k \in \ZZ^d}$ where $\phi_k(x) = e^{2\pi i k^Tx}$ and the Sobolev class $\Wcal(s,L)$ defined by:
\begin{eqnarray}
\Wcal(s,L) = \left\{f = \sum_{k \in \ZZ^d} a_k \phi_k \left| \sum_{k \in \ZZ^d} (\sum_{j=1}^d |k_j|^{2s})|a_k|^2 \le L\right.\right\}
\end{eqnarray}
In Lemma~\ref{lem:sob_holder} we show that the class $\Wcal(s',L')$ contains $\Sigma(s,L)$ as long as $s' < s$ and with appropriate choice of $L'$.
For now let us work in $\Wcal(s',L')$. 

Let us choose:
\begin{eqnarray*}
M_n = \{k \in \ZZ^d | |k_j| \le \frac{1}{2}m^{1/d}\}, \ \ 
m_0 = \left(18 \frac{d}{s'}2^{4s'/d}n^{-2}\right)^{\frac{-d}{4s'+d}} \asymp n^{\frac{2d}{4s'+d}}.
\end{eqnarray*}
Thinking of $M_n$ as an integer lattice with side lengths $m_0 = m^{1/d}$ we see that $|M_n| = m$.
Moreover $\sup_{i \notin M_n} |c_i|^4 = L^2 (2/m)^{4s'/d}$.
For the quadratic terms, this results in the bound:
\begin{eqnarray*}
\EE[(\hat{\theta}-\theta)^2] &\le& \|\psi\|_{\infty}^2\left( L^2(2/m)^{4s'/d} + 72 \|f\|_{\infty}^2 m/n^2 + 144 \|f\|_{\infty}^2/n\right)\\
& \le & \|\psi\|_{\infty}^2 \max\{1, \|f\|_{\infty}^2\}\max\{L^2,1\}\left((2/m)^{4s'/d} + 72m/n^2 + 144/n\right),
\end{eqnarray*}
and plugging in our definition of $m$ followed by some algebraic simplifications, we get
\begin{eqnarray*}
\EE[(\hat{\theta}-\theta)^2] \le 18 \|f\|_{\infty}^2\max\{1, \|p\|_{\infty}^2\} \max\{L^2, 1\} \left(\frac{8}{n} + n^{\frac{-8s'}{4s'+d}}\left[ 2^{\frac{8s'}{d}}d/s' + 3\right]\right).
\end{eqnarray*}

For the bilinear terms, plugging into Theorem~\ref{thm:bilinear_sub}, we get
\begin{eqnarray*}
\EE[(\hat{\theta} - \theta)^2] \le \|\psi\|_{\infty}^2 \max\{1, \|f\|_{\infty} \|g\|_{\infty}\} \max\{L^2,1\}\left((2/m)^{4s'/d} + m/n^2 + 3/n\right),
\end{eqnarray*}
which when we plug in for $m$ we get:
\begin{eqnarray*}
\EE[(\hat{\theta} - \theta)^2] \le \|\psi\|_{\infty}^2 \max\{1, \|f\|_{\infty} \|g\|_{\infty}\} \max\{L^2,1\}\left(3/n + n^{\frac{-8s'}{4s'+d}}\left[18 \times 2^{8s'/d} d/s' + 1\right]\right).
\end{eqnarray*}

\section{Proofs of Corollaries~\ref{cor:l2} and~\ref{cor:renyi}}
The proof of Corollary~\ref{cor:l2} is immediate given the decomposition $\|p-q\|_2^2 = \int p^2 + \int q^2 - 2\int pq$ and the Theorem~\ref{thm:quadratic_rate}.

For Corollary~\ref{cor:renyi}, if we use our estimator $\hat{T}$ for $T(p,q) = \int p^\alpha q^{1-\alpha}$ we can plug $\hat{T}$ into the definition of R\'{e}nyi divergence to obtain an estimator $\hat{D}_\alpha$. 
The rate of convergence is:
\[
\EE[|\hat{D}_\alpha - D_\alpha|] = \frac{1}{\alpha-1}\EE\left[\log\left( \hat{T}/T\right) \right] \le \frac{1}{\alpha-1} \EE\left[ \log(1 + |T - \hat{T}|/T)\right] \le \frac{1}{\alpha-1} cn^{-\gamma}/T(p,q)
\]
where $\gamma$ is the rate of convergence of our estimator.
This is $O(n^{-\gamma})$ as long as $T(p,q) \ge c > 0$.

%% file: app_lower.tex
\section{Detailed Proofs for Lower Bound}
To prove the main part of the theorem, the $\Omega(n^{\frac{-4s}{4s+d}})$ rate, we use Le Cam's method.
We decompose the proof into three parts.
In the first part, we adapt Le Cam's method to our setting.
In the second part, we show how the properties established on the functions $u_j$, $j \in [p]$ allow us to apply the technique and establish the theorem.
In the third part, we prove the existence of such functions $u_j$. 
We conclude this section with a proof of the $\Omega(n^{-1/2})$ when $s > d/4$.

\subsection{Proof of Lemma~\ref{lem:lecam}}
\begin{proof}
Define $\Theta_0 = \{g \in \Theta | T(g,q) \ge T(p,q)\}$ and $\Theta_1 = \{g \in \Theta | T(g,q) \le T(p,q) - 2\beta\}$ so that all $g_\lambda \in \Theta_1$ while $p \in \Theta_0$. 
Let $\tilde{\Theta}_i = \textrm{conv}(\{ G^n \times Q^n | g \in \Theta_i\})$ and consider the simple versus simple testing problem between $P \in \Theta_0$ and $G_\lambda \in \Theta_1$. 
The minimax probability of error $p_e$ of such a test is lower bounded by $\frac{1}{2} (1 - \sqrt{h^2(P, G_\lambda)(1 - h^2(P,G_\lambda))/4})$ by Theorem 2.2. of Tsybakov~\yrcite{tsybakov2009introduction}.
So for any test statistic $\psi$, taking supremum over $P \in \Theta_0, G \in \Theta_1$ we have:
\[
\sup_{\theta_{0,1} \in \tilde{\Theta}_{0,1}} p_e(\psi; \theta_0, \theta_1) \ge \frac{1}{2}\left[ 1 - \sqrt{\gamma(1-\gamma/4)}\right],
\]
where $\gamma \ge h^2(P^n\times Q^n, \bar{G}^n\times Q^n \in \tilde{\Theta}_1)$, which holds since $P^n\times Q^n \in \tilde{\Theta}_0$ and $\bar{G}^n\times Q^n \in \tilde{\Theta}_1$ by convexity. 
The same bound holds for after taking infimum over $\psi$.
Finally, if we make an error in the testing problem, we suffer loss at least $\beta$ which results in the statement in the Lemma.
\end{proof}

\subsection{The properties of $u_j$}
Recall that in our proof we partition $[0,1]^d$ into $m$ cubes $R_1, \ldots, R_m$ of side length $m^{-1/d}$.
On each bin we require a function $u_j$ such that:
\[
\textrm{supp}(u_j) \subset \{x | B(x,\epsilon) \in R_j\}, \ \|u_j\|_2^2 = \Theta(m^{-1}), \ \int_{R_j}u_j = 0, \ \int_{R_j}p^{\alpha-1}q^\beta u_j = 0, \|D^ru_j\|_{\infty} \le m^{r/d},
\]
where the last inequality needs to hold for all tuples $r$ with $\sum_j r_j \le s+1$. 
Using these functions $u_j$, we construct the alternatives $g_\lambda = p + K \sum_{\lambda \in \Lambda}\lambda_j u_j \mathbf{1}_{R_j}$ for all $\lambda \in \Lambda = \{-1,1\}^m$.
The third property above ensures that $g_\lambda$ is a valid density.

Properties 2, 4, and 5 ensure that $T(p,q) - T(g_\lambda, q)$ is sufficiently large. 
Indeed, by the von Mises expansion:
\begin{eqnarray*}
T(p,q) - T(g_\lambda, q) &=& K \alpha \sum_{j=1}^m \lambda_j \int_{R_j}p^{\alpha-1}q^\beta u_j + K^2\alpha(\alpha-1) \sum_{j=1}^m \int_{R_j} \xi_p^{\alpha-2}(x)q^\beta(x)u_j^2(x)dx\\
& \ge & c_0 K^2 \sum_{j=1}^m \|u_j\|_2^2 \ge c_1 K^2.
\end{eqnarray*}
Here $\xi$ is the function in the Taylor's remainder theorem, bounded between $p$ and $g_\lambda$, both of which are bounded above and below.
$g_\lambda$ is bounded above and below by property 5 since $\|D_0u_j\|_{\infty} = \|u_j\|_{\infty} \le 1$ which means that $g_\lambda \in [1-K, 1+K]$.
$K$ will be decreasing with $n$, so this quantity will certainly be bounded for $n$ large enough.
Property 2 allows us to arrive at the last line since each $u_j$ is orthogonal to the derivative of $T$, so the first term in the expansion is zero.
Finally property 4 allows us to lower bound $\|u_j\|_2^2$. 

Property 2 is also critical in ensuring that $h^2(P^n\times Q^n, \bar{G}^n \times Q^n)$ is small through the following Theorem of Birge and Massart~\yrcite{birge1995estimation}.
\begin{theorem}[\cite{birge1995estimation}]
Consider a set of densities $p$ and $p_\lambda = p[1+\sum_j \lambda_j v_j(x)]$ for $\lambda \in \Lambda = \{-1,1\}^m$. 
Suppose that (i) $\|v_j\|_{\infty} \le 1$ (ii) $\|1_{R_j^C}v_j\|_1 = 0$, (iii) $\int v_j p = 0$ and (iv) $\int v_j^2 p = \alpha_j > 0$ all hold with:
\[
\alpha = \sup_j \|v_j\|_{\infty}, \ s = n\alpha^2 \sup_j P(R_j), \ c = n\sup_j \alpha_j.
\]
Define $\bar{P}_\Lambda^n = \frac{1}{|\Lambda|}\sum_{\lambda \in \Lambda} P_\lambda^n$.
Then:
\[
h^2(P^n, \bar{P}_\Lambda^n) \le C(\alpha, s, c)n^2\sum_{j=1}^m \alpha_j^2,
\]
where $C < 1/3$ is continuous and non-decreasing with respect to each argument and $C(0,0,0) = 1/16$. 
\label{thm:bm_hellinger}
\end{theorem}

In bounding the Hellinger distance $h^2(P^n\times Q^n, \bar{G}^n \times Q^n)$ we first use the property that hellinger distance decomposes across product measures:
\[
h^2(P^n \times Q^n, \bar{G}^n \times Q^n) = 2 \left(1 - (1-h^2(P^n, \bar{G}^n)/2)(1-h^2(Q^n, Q^n)/2)\right) = h^2(P^n, \bar{G}^n).
\]
If we define $v_j(x) = Ku_j(x)/p(x)$ then we have $g_\lambda = p[1+\sum_j\lambda_j v_j]$ as needed by Theorem~\ref{thm:bm_hellinger}. 
We immediately satisfy requirements 1, 2, and 3 and we have $\int v_j^2 p = K^2 \int u_j^2/p \le K^2 \kappa_l/m = \alpha_j$.
Thus in applying the theorem we have:
\[
h^2(P^n \times Q^m, \bar{G}^n \times Q^m) \le (1/3) n^2 \sum_{j=1}^m \alpha_j^2 \le \frac{C n^2K^4}{m}.
\]

Property 1 and 5 ensure that $g_\lambda \in \Sigma(s,L)$ via the following argument. 
Defining $u_\lambda = K\sum_j \lambda_ju_j$, we will first show that $u_\lambda$ is holder smooth and $g_\lambda$ will be holder by a final application of the triangle inequality.
For $u_\lambda$, fix $r$ with $\sum_j r_j = s$ and fix $x,y$. 
Let $x_1$ be the boundary point of $R_j$, the bin containing $x$ along the line between $x$ and $y$ and let $y_1$ be the analogous boundary point for $y$.
\begin{eqnarray*}
|D^ru_\lambda(x) - D^ru_\lambda(y)| &\le& |D^ru_\lambda(x) - D^ru_\lambda(x_1)| + |D^ru_\lambda(x_1) - D^ru_\lambda(y_1)|+ |D^ru_\lambda(y_1) - D^ru_\lambda(y)|\\
& = & |D^ru_\lambda(x) - D^ru_\lambda(x_1)| + |D^ru_\lambda(y_1) - D^ru_\lambda(y)|\\
& = & \int_{\gamma(x,x_1)} \nabla D^r u_\lambda(z) dz + \int_{\gamma(y,y_1)} \nabla D^r u_\lambda(z) dz\\
& \le & K \|D^{r+1}u_j\|_\infty(\|x-x_1\|_2 + \|y-y_1\|_2)\\
& \le & K m^{(r+1)/d} \left( \|x-x_1\|_2^{s-r}\|x-x_1\|_2^{1-(s-r)} + \|y-y_1\|_2^{s-r}\|y-y_1\|_2^{1-(s-r)}\right)\\
& \le & K m^{(r+1)/d} \sqrt{d} m^{-\frac{1-(s-r)}{d}}\left( \|x - x_1\|_2^{s-r} + \|y-y_1\|_2^{s-r}\right) \\
& \le & K m^{s/d} \sqrt{d} \|x - y\|_2^{s-r} \le L \|x - y\|_2^{s-r}
\end{eqnarray*}
The first line is an application of the triangle inequality.
In the second line we use that $u_\lambda$ is zero and has all derivatives equal to zero on the boundaries of the cubes $R_j$. 
This follows from the fact that $u_j$ is not supported in the band around the border of $R_j$.
The third line is an application of the fundamental theorem of calculus, $\gamma(x,x_1)$ is the path between $x$ and $x_1$.
The fourth line follows from H\"{o}lder's inequality, we replace each derivative with its supremum and are left with just the path integral, which simplifies to the length of the path, i.e. $\|x-x_1\|_2$. 
In the fifth line we use the assumption $\|D^r u_j\|_{\infty} \le m^{r/d}$ for any derivative operator with $\sum_j r_j \le s+1$. 
To arrive at the sixth line, notice that since $x, x_1$ are in the same box $R_j$, we have $\|x - x_1\|_2 \le \sqrt{d} m^{-1/d}$ (there are $m$ boxes and each one has length $m^{-1/d}$ on each side). 
The last line is true since $x_1,y_1$ are on the line segment between $x,y$. 

In other words, $g_\lambda$ is holder smooth as long as $Km^{s/d}\sqrt{d} \asymp L$, imposing the requirement that $K = O(m^{-s/d})$. 
So if we pick $m = n^{\frac{2d}{4s+d}}$ and $K = m^{-s/d} = n^{\frac{-2s}{4s+d}}$ we get that $g_\lambda \in \Sigma(s,L)$ as long as there is some wiggle room around $p$.
We also get that the Hellinger distance is bounded by $O( n^2 n^{\frac{-8s}{4s+d}} n^{\frac{-2d}{4s+d}}) = O(1)$ and the distance in our metric is $n^{\frac{-4s}{4s+d}}$ as we desired. 
We can apply Theorem~\ref{lem:lecam} and arrive at the result. 

\subsection{Existence of $u_j$}
To wrap up, we need to show that we can in fact find the functions $u_j$. 
We can do this by mapping $R_j$ to $[0,1]^d$ and using an orthonormal system $\{\phi_j\}_{j=1}^q$ for $L^2([0,1]^d)$ with $q \ge 3$.
Suppose that  $\phi_j$ satisfy (i) $\phi_1 = 1$, $\phi_j(x) = 0$ for $x \notin [\epsilon, 1-\epsilon]^d$ and (iii) $\|D^r\phi_j\|_{\infty} \le K < \infty$ for all $j$.
Certainly we can find such an orthonormal system.

Now for any function $f \in L^2([0,1]^d)$, we can easily find a unit-normed function $\tilde{v} \in \textrm{span}(\{\phi_j\})$ such that $\tilde{v} \perp \phi_1$, and $\tilde{v} \perp f$. 
If we write $\tilde{v} = \sum_i c_i \phi_i$ we have that $D^r v = c_i D^r \phi_i$ so that $\|D^rv\|_{\infty} \le K \sum_i |c_i| \le K \sqrt{q}$ since $\tilde{v}$ is unit-normed. 
Notice that the vector $v = \tilde{v}(K\sqrt{q})^{-1}$ has upper and lower-bounded $\ell_2^2$-norm while having all $\|D^rv\|_{\infty} \le 1$.

To construct the functions $u_j$, map the $R_j= \Pi_{i=1}^d[j_im^{-1/d}, (j_i+1)m^{-1/d}]$ to $[0,1]^d$ and let the function $f = p^{\alpha-1}(x)q^\beta(x)$ mapped appropriately to $[0,1]^d$. 
Use the function $v_j$ constructed in the previous paragraph.
In mapping back to $R_j$, let $u_j(x) = v_j(m^{1/d}(x - (j_1, \ldots, j_d))^T)$ so that $\int_{R_j} u_j^2(x)dx = m^{-1} \int v_j^2(x)dx = \Omega(1/m)$ and $\|D^r u_j\|_{\infty} \le m^{r/d}$.
These functions $u_j$ meet the requirements 1-5 outlined above, allowing us to apply Le Cam's method.

\subsection{An $n^{-1/2}$ Lower Bound when $s > d/4$}
To obtain the $n^{-1/2}$ lower bound for the highly-smooth setting, we will reduce the problem of estimating $T(p,q)$ to that of estimating a quadratic functional of the two densities:
\begin{eqnarray}
\theta(p,q) = \int a_1(x)p(x) + a_2(x)q(x) + a_3(x) p(x)q(x) + a_4(x)p^2(x) + a_5(x)q(x) d \mu(x)
\label{eq:quadratic_functional}
\end{eqnarray}
for some known functions $a_i: [0,1]^d \rightarrow \RR$, $i \in \{1, \ldots, 5\}$. 
We will then use the following lower bound on the rate of estimating these functionals to establish a lower bound in our problem:
\begin{theorem}
Let $a_i:[0,1]^d \rightarrow \RR, i \in \{1, \ldots, 5\}$ be continuous, bounded, non-constant functions and let $\theta(p,q)$ be as in Equation~\ref{eq:quadratic_functional}.
Then:
\begin{eqnarray}
\liminf_{n \rightarrow \infty} \inf_{\hat{\theta}_n} \sup_{p,q \in \Sigma(s,L)}\PP_{X_1^n \sim p,Y_1^n \sim q}[|\hat{\theta}_n - \theta(p,q)| \ge \epsilon n^{-1/2}] \ge c > 0
\end{eqnarray}
For some constants $\epsilon,c > 0$. 
\label{thm:linear_lower_bound}
\end{theorem}
\begin{proof}
We will use Le Cam's Method to establish the lower bound. 
Let us fix $q$ once and for all. We will only vary $p$.
Let $p_0(x) = 1$ and $p_1(x) = 1+u(x)$ for some function $u(x)$ that we will select later.
By Theorem 2.2 of~\cite{tsybakov2009introduction} (essentially the Neyman-Pearson Lemma) if we can upper bound $KL(p_1^n \times q^n, p_0^n \times q^n)$ we have a lower bound on the probability of making an error in the simple versus simple hypothesis test between the two possible distributions when $X_1^n,\sim p_1$ and $Y_1^n \sim q$. 
Mathematically, define $p_{e,1}(\psi) = \PP_{X_1^n \sim p_1, Y_1^n \sim q} [ \psi(X_1^n, Y_1^n) \ne 1]$ for a test statistic $\psi$ taking values in $\{0,1\}$. 
Also define $p_{e,1} = \inf_{\psi} p_{e,1}(\psi)$.
Then Theorem 2.2 of~\cite{tsybakov2009introduction} says that if $KL(p_1^n \times q^n, p_0^n \times q^n) \le \alpha < \infty$ then
\[
p_{e,1} \ge \max\left(\frac{1}{4}\exp(-\alpha), \frac{1-\sqrt{\alpha/2}}{2}\right)
\]
So let us bound the KL-divergence:
\begin{eqnarray*}
KL(p_1^{n}\times q^n, p_0^{n} \times q^n) = nKL(p_1, p_0) = n \int (1+u(x)) \log(1+u(x)) dx \le n \int u(x) +u^2(x) dx = n \|u\|_2^2
\end{eqnarray*}
Here we used that $\int u(x) = 0$ if $p_1$ is to remain a density. 
This is one of the requirements on the function $u$ that we will pick.
If the KL-divergence is to remain bounded, we will also require that $\|u\|_2^2 \le c/n$ for some constant.

If we make a mistake in the testing problem, we suffer at least $1/2 |\theta(p_0,q) - \theta(p_1, q)|$ loss in the estimation problem.
So we must lower bound the absolute difference between the two functional values.
\begin{eqnarray*}
|\theta(p_0, q) - \theta(p_1, q)| &=& | \int a_1(x)u(x) + a_3(x)q(x)u(x) + 2 a_4(x) u(x) + a_4(x) u^2(x) d\mu(x)|\\
& = & |\int f(x) u(x) + a_4(x) u^2(x) d\mu(x)|
\end{eqnarray*}
where $f(x) = a_1(x) + a_3(x)q(x) + 2a_4(x)$.
Suppose we had a function $v$ such that:
\[
 \int v(x) = 0, \  \|v(x)\|_2^2 = O(1),  \ p_1 = 1+1/\sqrt{n} v(x) \in \Sigma(s,L), \  \int f(x)v(x) = \Omega(1)
\]
Then if we use $u(x) =n^{-1/2} v(x)$ the loss we suffer is at least $c_1/\sqrt{n} - c_2/n \ge \epsilon n^{-1/2}$ for some $\epsilon > 0$ for $n$ sufficiently large. 
At the same time, the KL-divergence between the two hypothesis is also $O(1)$.
So we would be able to apply Le Cam's inequality.

So, we just need to find a sufficiently smooth function $v$ with constant $\ell_2^2$ norm and constant inner product with $f$.
To do this, consider an orthonormal system $\phi_1, \ldots, \phi_q$ with $q \ge 3$ of $L^2([0,1]^d)$ such that (i) $\phi_j(x) = 1$, (ii) $f \in \textrm{span}(\{\phi_j\}_{j=1}^q)$ and (iii) $\|D^r\phi_j\|_{\infty} \le K < \infty$ for all $j$ and all tuples $r$ with $\sum_j r_j \le s + 1$. 
It is always possible to construct such a system as long as $f$ itself has bounded $r$-th derivatives, which is true since $f$ itself is a continuous, bounded function over a compact domain. 
Let $L$ denote the linear space spanned by $\{\phi_j\}$. 
Earlier we showed that if $v \in L$, then $v \in \Sigma(s, A)$ for sufficiently large constant $A$. 
So we can let $v$ be any unit-normed function in $L' = \{v \in L | \langle v, f \rangle = c, \langle v, \phi_1 \rangle = 0\}$, which is an affine space of dimension at least $1$ (since $f \ne c \phi_1$).

Then $u(x) = v(x)/\sqrt{n}$ meets all of the requirements.
Notice that since $v \in \Sigma(s,A)$, we have that $u \in \Sigma(s, A/\sqrt{n}) \subset \Sigma(s, L)$ for $n$ sufficiently large. 
\end{proof}


In what follows, the functional $\theta$ that we are trying to estimate will actually be a random quantity. 
However, since Theorem~\ref{thm:linear_lower_bound} applies to any set of five bounded continuous function $a_1, \ldots, a_5$, it actually applies to any distribution over this space of five bounded continuous functions.
Mathematically, for any distribution $\Dcal$ over this space of bounded continuous functions:
\[
\liminf_{n \rightarrow \infty} \inf_{\hat{\theta}_n} \sup_{p,q \in \Sigma(s,L)} \PP_{X_1^n \sim p, Y_1^n \sim q, (a_1, \ldots, a_5) \sim \Dcal} \left[ |\hat{\theta}_n(a_1^5) - \theta(a_1^5,p,q)| \ge \epsilon n^{-1/2}\right] \ge c > 0
\]
where $\theta(a_1^5,p,q)$ is given in Equation~\ref{eq:quadratic_functional}.

Let us use Theorem~\ref{thm:linear_lower_bound} to prove a lower bound for estimating $T(p,q) = \int p^\alpha q^\beta$. 
Suppose we had an estimator $\widehat{T}_n$ for $T(p,q)$ that converges at rate $o(n^{-1/2})$, say $\forall p, q, n, \EE[|\widehat{T}_n - T(p,q)|] \le c_1n^{-1/2-\epsilon}$ for some constants $c_1, \epsilon > 0$. 
We will use it to construct an estimator for a quadratic functional of $p,q$ with better-than-$\sqrt{n}$ rate, which will contradict Theorem~\ref{thm:linear_lower_bound}.

The quadratic functional of $p,q$ will be the terms in the second order expansion of $T(p,q)$ about $T(\phat_n, \qhat_n)$. 

Given $2n$ samples, as in our upper bound, we use the first $n$ to construct estimators $\phat_n, \qhat_n$ for $p,q$ respectively. 
We use the second $n$ samples to compute $\hat{T}_n$. 
The estimator for $\theta$ will be $\hat{\theta}_{2n} = \widehat{T}_n - C_2 T(\phat_n,\qhat_n)$.
Where we are collecting all of the terms of the form $T(\phat_n, \qhat_n)$ together.
Recall that $C_2$ is the coefficient for all of these terms.

The risk of the estimator is:
\begin{eqnarray*}
\EE_{X_1^{2n}}[|\hat{\theta}_n - \theta|] &\le& \EE_{X_{n+1}^{2n}}[|\widehat{T}_n - T|] + \EE_{X_1^{2n}}[|T - C_2 T(\phat,\qhat) - \theta|]\\
& \le & c_1n^{-1/2-\epsilon} + O(\EE_{X_1^n}[\|p - \phat\|_3^3 + \|q -\qhat\|_3^3])\\
& \le & c_1n^{-1/2-\epsilon} + c_2n^{\frac{-3s}{2s+d}}
\end{eqnarray*}
for constants $c_1, c_2 > 0$.
Now if $s > d/4$, both terms are $o(n^{-1/2})$, so we have $\EE[|\hat{\theta}_n - \theta|] = o(n^{-1/2})$.
The functions $\phat_n,\qhat_n$ are deterministic functions of $X_1^n, Y_1^n$, so we can think of $X_1^n$ as encoding a distribution over functions $\phat_n,\qhat_n$. 

More formally, let $\Dcal$ encode the following distribution: We drawn $X_1^n, Y_1^n$ from $p,q$ respectively and compute $\phat_n, \qhat_n$.
With these, the five functions $a_1, \ldots, a_5$ are:
\begin{eqnarray*}
a_1 &=& \alpha (2 - \alpha - \beta)\phat_n^{\alpha-1} \qhat_n^\beta\\
a_2 &=& \beta (2 - \alpha - \beta) \phat_n^\alpha\qhat_n^{\beta-1}\\
a_3 &=& \alpha\beta \phat_n^{\alpha-1} \qhat_n^{\beta-1}\\
a_4 &=& 1/2 \alpha(\alpha-1) \phat_n^{\alpha-2}\qhat_n^\beta\\
a_5 &=& 1/2 \beta(\beta-1) \phat_n^\alpha \qhat_n^{\beta-2}
\end{eqnarray*}
Notice that all of these functions are continuous and they can be bounded from above and below if we use the truncated kernel density estimators.
Now whenever $s> d/4$:
\[
\EE_{(a_1, \ldots, a_5) \sim \Dcal} \EE_{X_{1}^n \sim p, Y_1^n \sim q} \left[ |\hat{\theta} - \theta|\right] = \EE_{X_1^{2n} \sim p, Y_1^{2n} \sim q}\left[|\hat{\theta} - \theta|\right] \le cn^{-1/2-\epsilon}
\]
which contradicts the lower bound.
Via Markov's inequality, $\PP_{X_1^{2n}}[ |\hat{\theta}_n - \theta| \ge c_4 n^{-1/2}] \le o(n^{-1/2})/n^{-1/2} \rightarrow 0$ which contradicts our discussion following Theorem~\ref{thm:linear_lower_bound}.
This shows that when $s > d/4$, one cannot estimate $T(p,q)$ are faster than $\sqrt{n}$ rate. 

\subsection{Translating to $T_\alpha$ and $D_\alpha$}
Suppose we have an estimator $\hat{S}_\alpha$ for the Tsallis-$\alpha$ divergence, such that for all $p,q \in \Sigma(s,l) \EE[|\hat{S}_\alpha - S_\alpha|] \le \epsilon_n$.
We can define an estimator $\hat{T}$ for $T(p,q) = \int p^\alpha q^{1-\alpha}$ as $\hat{T} = (\alpha-1)\hat{S}_\alpha + 1$. 
The error between $\hat{T}$ and $T$ is:
\[
\EE[|\hat{T} - T|] = |\alpha-1| \EE[|\hat{S}_\alpha - S_\alpha|] \le |\alpha-1| \epsilon_n
\]
We therefore know that $\epsilon_n = \Omega(n^{-\gamma})$ where $\gamma = \min\{\frac{4s}{4s+d}, 1/2\}$ since otherwise we would have an estimator $\hat{T}$ for $T(p,q)$ with rate $o(n^{-\gamma})$, which contradicts Theorem~\ref{thm:lower_bound}.

For $D_\alpha$, we use the same proof structure, but computing the error for $\hat{T}$ is more involved. 
The estimator $\hat{T} = \exp\{(\alpha-1)\hat{D}\}$ has error:
\[
\EE[|\hat{T} - T|] = \EE\left[ |\exp\{(\alpha-1)\hat{D}\} - \exp\{(\alpha-1)D_\alpha\}|\right]
\]
We would like to eliminate the absolute value, so we will have to consider all of the cases.
If $\alpha < 1$ and $D > \hat{D}$ then the first term dominates the second so we can simply drop the absolute value sign. 
In this case we can use convexity of $e^x$ to upper bound by:
\[
\le (\alpha-1) \EE[e^{(\alpha-1)\hat{D}} (\hat{D} - D_\alpha)] = (1-\alpha) \EE[e^{(\alpha-1)\hat{D}} (D_\alpha - \hat{D})] \le C \epsilon_n
\]
as long as $D_\alpha$ is bounded from below, which implies that for $n$ large enough, $e^{(\alpha-1)\hat{D}} = O(1)$.
Actually the other cases are analogous, for example if $\hat{D} > D$, then to remove the absolute value, we must swap the two terms, after which we can use convexity to arrive at the same upper bound.
Thus we have shown that $\EE[|\hat{T} - T|] = O(\epsilon_n)$ which implies that $\EE[|\hat{D} - D|] = \Omega(n^{-\gamma})$ as claimed.

%% file: app_aux.tex
\section{More Auxiliary Results}

\begin{lemma}[H\"{o}lder is contained in Sobolev]
Let $f \in \Sigma(s, L)$ belong to the periodic holder class with smoothness $s$. 
Then $f$ belongs to the sobolev ellipsoid $\Wcal(s', L')$ where $\phi_k(x) = e^{2i \pi k^Tx}$ is the fourier basis, $k \in \ZZ^d$, $s' < s$ and:
\[
L' = \frac{d CL^2}{(2 \pi)^{2\lfloor s \rfloor}}
\]
with $C = \sum_{l=0}^{\infty} 4^{l(s'-s)}$.
\label{lem:sob_holder}
\end{lemma}
\begin{proof}
Let us decompose $s = r+\alpha$ where $r = \lfloor s \rfloor$ and $\alpha \in (0,1]$. 
We need to bound:
\[
\sum_{(k_1,\ldots, k_d) \in \ZZ^d} (\sum_{j=1}^d |k_j|^{2s'}) |\alpha_k|^2
\]
where $\alpha_k = \int f(x) \phi_k(x)dx$.
This is equivalent to bounding, for each $j = [d], \sum_{k \in \ZZ^d} |k_j|^{2s'} |\alpha_k|^2$ so let us fix a dimension $j$ for now.
Using repeated integration by parts and the fact that $D^{\vec{r}} f$ is period for all $\vec{r}$ with $\sum_j r_j \le r$. we get
\begin{eqnarray*}
\left|\int \frac{\partial^r}{\partial x_j^r}f(x)\phi_k(x) dx\right| = |2 \pi i k_j|^r |\int f(x) \phi_k(x) dx| = |2 \pi i k_j|^r |\alpha_k|
\end{eqnarray*}
Let us write $g(x) = \frac{\partial^r}{\partial x_j^r}f(x)$. 
Then since $f \in \Sigma(s, L)$, we know that $g$ satisfies:
\[
|g(x) - g(y)| \le L \|x - h\|^\alpha
\]
for all $x,y$. 
We will use this fact to bound $\sum_{k \in \ZZ^d} |k_j|^{2\alpha'} |b_k|^2$  where $b_k = \int g(x) \phi_k(x)$ and $\alpha' < \alpha$ which will give us a bound on $\sum_{k \in \ZZ^d} |k_j|^{2s'} |\alpha_k|$ via the above calculation.
In particular, suppose that $\sum_{k \in ZZ^d} |k_j|^{2 \alpha'} |b_k|^2 \le \gamma_j$, then:
\[
\sum_{k \in \ZZ^d} |k_j|^{2s'} |\alpha_k|^2 = \sum_{k \in \ZZ^d} |k_j|^{2r+2\alpha'}|\alpha_k^2| = |2\pi i|^{-2 r} \sum_{k \in \ZZ^d} |k_j|^{2 \alpha'} |b_k|^2 \le (2\pi)^{-2 r} \gamma_j
\]

Notice that:
\[
g(x_1, \ldots, x_j-h, \ldots, x_d) - g(x_1, \ldots, x_j+h, \ldots, x_d) = \sum_{k \in \ZZ^d} b_k e^{2 i \pi k^Tx} 2i \sin(2\pi k_j h)
\]
This means that:
\begin{eqnarray*}
4 \sum_{k \in \ZZ^d} |b_k|^2 \sin^2(2 \pi k_j h) = \int (g(x_1, \ldots, x_j-h, \ldots, x_d) - g(x_1, \ldots, x_j+h, \ldots, x_d))^2 dx \le L^2 |h|^{2\alpha}
\end{eqnarray*}
Notice that $\sin^2(\pi/2) > \sin^2(\pi/4) \ge 1/2$ so if we pick $h = 1/(8q)$ and $k_j \in \{q, \ldots, 2q-1\} \cup \{-q, \ldots, -2q+1\}$ we can lower bound the left hand side. 
To be concrete, letting $S_q = \{k \in \ZZ^d | k_j \in \{q, \ldots, 2q-1\} \cup \{-q, \ldots, -2q + 1\}\}$:
\begin{eqnarray*}
\sum_{k \in \ZZ^d} |b_k|^2 |k_j|^{2\alpha'} = \sum_{l = 0}^{\infty} \sum_{k \in S_{2^l}} |b_k|^2 |k_j|^{2\alpha'} \le \sum_{l=0}^{\infty} (2^{l+1})^{2\alpha'} \sum_{k\in S_{2^l}} |b_k|^2 
\end{eqnarray*}
But:
\begin{eqnarray*}
\sum_{k \in S_{2^l}} |b_k|^2 \le 2 \sum_{k \in S_{2^l}} |b_k|^2 \sin^2(2 \pi k_j (1/2^{l+3})) \le 2 \sum_{k \in \ZZ^d} |b_k|^2 \sin^2 (2 \pi k_j (1/2^{l+3})) \le \frac{L^2}{2} 2^{-2\alpha (l+3)}
\end{eqnarray*}
Using this bound above, we get:
\[
\sum_{k \in \ZZ^d} |b_k|^2|k_j|^{2\alpha'} \le \frac{L^2}{2} \frac{4^{2\alpha'}}{8^{2 \alpha}} \sum_{l=0}^{\infty} 4^{l(\alpha' - \alpha)} \le CL^2
\]
whenever the series converges (as long as $\alpha' < \alpha$).

Using this as our value for $\gamma_j$ and summing over the $d$ dimensions, we get:
\[
\sum_{j=1}^d \sum_{k \in \ZZ^d} |k_j|^{2s'} |\alpha_k| \le d (2\pi)^{-2 r} \gamma_j \le \frac{d CL^2}{(2 \pi)^{2r}}
\]
\end{proof}